\documentclass{article}

\usepackage{microtype}
\usepackage{graphicx}
\usepackage{subfigure}
\usepackage{booktabs} 

\usepackage{hyperref}

\usepackage[accepted]{icml2025}

\usepackage{amsmath}
\usepackage{amssymb}
\usepackage{mathtools}
\usepackage{amsthm}

\usepackage[capitalize,noabbrev]{cleveref}

\theoremstyle{plain}
\newtheorem{theorem}{Theorem}[section]

\newtheorem{lemma}[theorem]{Lemma}

\theoremstyle{definition}

\newtheorem{assumption}[theorem]{Assumption}
\theoremstyle{remark}

\usepackage[textsize=tiny]{todonotes}

\icmltitlerunning{Distributed Deep Learning based on Stochastic Gradient Staleness}

\begin{document}

\twocolumn[
\icmltitle{Distributed Deep Learning using Stochastic Gradient Staleness}



\icmlsetsymbol{equal}{*}

\begin{icmlauthorlist}
\icmlauthor{Viet Hoang Pham}{yy}
\icmlauthor{Hyo-Sung Ahn}{yyy}
\end{icmlauthorlist}
\icmlaffiliation{yy}{Information Technology Department, Posts and Telecommunications Institute of Technology (PTIT), Hanoi, Vietnam. Email: vietph@gm.gist.ac.kr.}

\icmlaffiliation{yyy}{Department of Mechanical and Robotics Engineering, Gwangju Institute of Science and Technology, Gwangju, Korea.}

\icmlcorrespondingauthor{Hyo-Sung Ahn}{hyosung@gist.ac.kr}

\icmlkeywords{Distributed Training, GNN}

\vskip 0.3in
]



\printAffiliationsAndNotice{}  
\begin{abstract}
Despite the notable success of deep neural networks (DNNs) in solving complex tasks, the training process still remains considerable challenges. A primary obstacle is the substantial time required for training, particularly as high-performing DNNs tend to become increasingly deep (characterized by a larger number of hidden layers) and require extensive training datasets. To address these challenges, this paper introduces a distributed training method that integrates two prominent strategies for accelerating deep learning: data parallelism and fully decoupled parallel backpropagation algorithm. By utilizing multiple computational units operating in parallel, the proposed approach enhances the amount of training data processed in each iteration while mitigating locking issues commonly associated with the backpropagation algorithm. These features collectively contribute to significant improvements in training efficiency. The proposed distributed training method is rigorously proven to converge to critical points under certain conditions. Its effectiveness is further demonstrated through empirical evaluations, wherein an DNN is trained to perform classification tasks on the CIFAR-10 dataset.
\end{abstract}
\section{Introduction}\label{sec_intro}
Deep neural networks (DNNs) have emerged as a prominent approach in both academic research and industrial applications. They have demonstrated significant efficacy in addressing complex tasks, encompassing many areas ranging from image classification \cite{YanshengLi2021}, speech processing \cite{AmbujMehrish2023}, to medical applications \cite{BalazsHarangi2018, ShavetaDargan2020} and the robotics field \cite{ArturIstvanKaroly2021}. The remarkable achievements of DNNs are largely attributable to their deeply layered architectures, which enhance the robustness and performance of models in practical implementations. Additionally, training a DNN typically requires a large amount of training data to ensure the model's generalization. The incorporation of multiple hidden layers within the network, coupled with the reliance on extensive training datasets, inevitably results in increased computational demands and prolonged training time.
Most training methods are based on the stochastic gradient descent (SGD) scheme \cite{LeonBottou2018}. At each iteration, a mini-batch of training data is randomly selected, and parameter updates are executed based on the gradient of the cost function corresponding to the selected mini-batch. While employing the cost function of a mini-batch instead of the full dataset increases the number of iterations required for convergence, it substantially reduces the computational time per iteration. Consequently, SGD is considerably more efficient compared to methods utilizing full gradient information. However, with a constant step size, SGD ensures convergence only to a neighborhood of critical points. The backpropagation (BP) algorithm \cite{DavidERumelhart1986} is the standard technique for calculating gradients of cost functions. It necessitates performing a forward pass and a backward pass before parameter updates. These passes must be executed sequentially, layer by layer, which introduces a bottleneck and represents a significant source of time inefficiency. To overcome the limitations of SGD and BP, various parallel and distributed computing strategies \cite{MichaelDiskin2021} have been proposed to facilitate more efficient and timely training processes. These strategies can be classified into model parallelism, data parallelism, pipeline parallelism, and hybrid parallelism.

Model parallelism represents a parallelization strategy wherein a large neural network model is divided into multiple modules, enabling concurrent operations across these modules while utilizing the same mini-batch. This approach facilitates the acceleration of DNN training by assigning distinct parameters and computations to each module, corresponding to specific layers of the model. \cite{ZhouyuanHuo2018} proposed the Decoupled Parallel Backpropagation framework, which mitigates the issue of backward locking by employing delayed error gradients. Extending this concept, \cite{HuipingZhuang2022} introduced a fully decoupled scheme, leveraging delayed gradient techniques to address locking challenges in both forward and backward passes.
In data parallelism, the training dataset is partitioned into subsets and distributed across multiple workers. Each worker maintains a complete replica of the DNN and computes the gradients based on its respective subset of training data. This approach increases the amount of training data processed during each iteration, thereby mitigating the adverse effects associated with the SGD. The global DNN model is constructed through the aggregation of model replicas and/or the gradients computed by the workers.
This aggregation is typically implemented using two primary architectures: the parameter server architecture and the decentralized architecture. The parameter server architecture \cite{ZhenSong2021} is a centralized framework wherein all workers communicate with a server responsible for aggregating gradients and updating the model. In contrast, the decentralized architecture \cite{MahmoudAssran2019} requires communication among neighboring workers.
Pipeline parallelism divides the computational task, including both data and model components, into a series of processing stages. Each stage processes its input from the preceding stage and immediately forwards the results to the subsequent stage \cite{SunwooLee2017, ChiheonKim2020}. However, the inherent dependency between stages imposes constraints on the scalability of this approach. Several studies incorporate elements of both data and model parallelism methods \cite{LinghaoSong2019, JunyaOno2019}, though these efforts primarily target specific neural network architectures without providing comprehensive analysis or generalizable insights.

This paper aims to design a distributed training method that integrates the prominent advantages of data parallelism and model parallelism. Both modules in model parallelism and workers in data parallelism are treated as computational units. In our proposed framework, these computational units are collectively referred to as agents, following the terminology of multi-agent systems. In alignment with data parallelism, the training data is partitioned into multiple subsets and distributed across different groups of agents for parallel processing. These groups are termed data-groups. The communication among data-groups follows the decentralized architecture. Agents within each data-group collaborate to execute the fully decoupled parallel backpropagation algorithm \cite{HuipingZhuang2022}, thereby enhancing computational efficiency during each iteration. Since multiple agents estimate parameters for the same part of the neural network model, they are organized into a model-group. A consensus scheme is applied to each model-group to ensure that the parameters of the corresponding model part achieve consistency across the agents.
The main contributions of this paper are as follows:
\begin{itemize}
\item We integrate data parallelism with the fully decoupled parallel backpropagation algorithm to design a distributed training method that accelerates deep learning without compromising accuracy.
\item We provide convergence analysis of the proposed method, demonstrating its convergence to critical points under common assumptions in deep learning.
\item We conduct experiments by training a deep neural network for a classification task to validate that the proposed distributed training method achieves significant speedup compared to centralized settings.
\end{itemize}
The remainder of the paper is outlined as follows. Section 2 presents the background on deep learning and communication protocols in multi-agent systems. The algorithm is described in detail in Section 3. Convergence analysis is provided in Section 4. Section 5 reports the experimental results to assess the performance of the proposed training method, while Section 6 concludes the paper.

\subsection*{Notations}
We use $\mathbb{R}$ and $\mathbb{R}^{m \times n}$ to denote the set of real numbers and the set of $m \times n$ matrices, respectively.
Let $\textbf{1}_n$ be the vector in $\mathbb{R}^{n}$ whose all elements are $1$ and $\textbf{I}_n$ be the identity matrix in $\mathbb{R}^{n \times n}$. When the dimensions are clear, the subscripts of the matrix $\textbf{I}_n$ and the vector $\textbf{1}_n$ can be removed.
Denote by $\textbf{H}^T$ the transpose matrix of $\textbf{H}$ and $\otimes$ by the Kronecker product. 
We use $\mathbb{E}\left[X\right]$ to denote the expected value of a random variable $X$.
For a given set $\mathcal{H}$, $|\mathcal{H}|$ represents the cardinality of this set. 
When the set $\mathcal{H}$ has a finite number of vectors, i.e., $\mathcal{H} = \{\textbf{a}_1, \textbf{a}_2, \dots, \textbf{a}_n\}$, we use $col \mathcal{H}$ to define the following column vector 
\[col \mathcal{H} = col \{\textbf{a}_1, \textbf{a}_2, \dots, \textbf{a}_n\} = [\textbf{a}_1^T, \textbf{a}_2^T, \dots, \textbf{a}_n^T]^T.\]
\section{Background}\label{sec_bkg}
This section first revisits some necessary knowledge for DNN training. The backpropagation algorithm as well as  forward, backward,
and update lockings \cite{MaxJaderberg2017} are summarized.
At the end of this section, we introduce some basic notation of graph theory, which are usually used to describe the communication and coordination in multiagent system.
\subsection{DNN training}\label{subsec_SGD}
Suppose that we need to train a $L$-layer DNN. Each layer $l, \forall 1 \le l \le L$, produces an activation vector $\textbf{h}_l = \mathcal{A}_l(\textbf{h}_{l-1},\textbf{w}_l)$ taking an input $\textbf{h}_{l-1}$ with the weight vector $\textbf{w}_l$. Here, we use $\textbf{w}_l$ to denote the vectorization of the weight in the layer $l$.
Denote by $d_l$ as the dimension of weight in the layer $l$, i.e., $\textbf{w}_l \in \mathbb{R}^{d_l}$. The parameter vector of the entire DNN is $\overline{\textbf{W}} = \left[\begin{matrix}\textbf{w}_1^T & \textbf{w}_2^T & \cdots & \textbf{w}_L^T\end{matrix}\right]^T \in \mathbb{R}^{\sum_{l=1}^{L}d_l}$.
Thus, the output of the DNN can be represented as $\textbf{h}_L = \mathcal{A}(\textbf{h}_0,\overline{\textbf{W}}) = \mathcal{A}_L(\mathcal{A}_{L-1}(\cdots (\mathcal{A}_2(\mathcal{A}_1(\textbf{h}_0,\textbf{w}_1), \textbf{w}_2)\cdots),\textbf{w}_{L-1}),\textbf{w}_L)$ where $\textbf{h}_0$ is the input of the DNN.
We use $\boldsymbol{\chi}^{(n)} = (\boldsymbol{\xi}^{(n)}, \boldsymbol{\nu}^{(n)})$ to denote a sample of training data. Here, $\boldsymbol{\nu}^{(n)}$ is target when $\boldsymbol{\xi}^{(n)}$ is used as the input of the DNN.
Let $\phi$ be the function used to measure the error between the outputs of the DNN and the targets. Then the loss corresponding to one training data point $\boldsymbol{\chi}^{(n)}$ is \[\phi(\boldsymbol{\chi}^{(n)},\overline{\textbf{W}}) = \phi(\mathcal{A}(\boldsymbol{\xi}^{(n)},\overline{\textbf{W}}),\boldsymbol{\nu}^{(n)}).\]
Given a training dataset $\mathcal{D} = \{\boldsymbol{\chi}^{(n)}: 1 \le n \le N\}$, the total loss function of the training process is
\begin{equation}
\Psi(\overline{\textbf{W}}) = \frac{1}{N}\sum\limits_{\boldsymbol{\chi}^{(n)} \in \mathcal{D}}\phi(\boldsymbol{\chi}^{(n)},\overline{\textbf{W}}).
\end{equation}
Then the training problem is given as the following optimization problem:
\begin{equation}\label{eq_training_problem}
\min\limits_{\overline{\textbf{W}}} \Psi(\overline{\textbf{W}})
\end{equation}

Stochastic gradient based methods are often used for deep learning optimization. In iteration $t$, a mini-batch $\mathcal{B}(t) \subset \mathcal{D}$ is sampled with the cardinality $|\mathcal{B}(t)| = B$. The weight of the DNN is updated as
\begin{equation}\label{eq_SGD}
\overline{\textbf{W}}(t+1) = \overline{\textbf{W}}(t) - \eta_t\nabla\Phi(t)
\end{equation}
where $\eta_t$ is the step size and $\Phi(t) = \frac{1}{B}\sum\limits_{\boldsymbol{\chi}^{(n)} \in \mathcal{B}(t)}\phi(\boldsymbol{\chi}^{(n)},\overline{\textbf{W}}(t))$ is the loss function corresponding to the mini-batch $\mathcal{B}(t)$.
The update law \eqref{eq_SGD} can be written in the more detailed form as
\begin{equation}
\textbf{w}_l(t+1) = \textbf{w}_l(t) - \eta_t\frac{1}{B}\sum\limits_{\boldsymbol{\chi}^{(n)} \in \mathcal{B}(t)}\frac{\partial \phi(\boldsymbol{\chi}^{(n)},\overline{\textbf{W}}(t))}{\partial \textbf{w}_l}
\end{equation}
for every layer $l \in [1, L]$. The notation $\frac{\partial \phi(\boldsymbol{\chi}^{(n)},\overline{\textbf{W}})}{\partial \textbf{w}_l}$ describes the gradient of the loss function $\phi(\boldsymbol{\chi}^{(n)},\overline{\textbf{W}}(t))$ with respect to the weight vector $\textbf{w}_l$.
\subsection{Backpropagation algorithm}\label{subsec_backpropagation}
To compute $\frac{\partial \phi(\boldsymbol{\chi}^{(n)},\overline{\textbf{W}})}{\partial \textbf{w}_l}$, we usually use the backpropagation algorithm consisting of two passes of the DNN: the forward pass and the backward pass.
In the forward pass, activation values of the layer $l, \forall 1 \le l \le L$, depends not only its weight vector $\textbf{w}_l$ but also the activation from its previous layer.
\begin{equation}\label{eq_forward_step}
\textbf{h}_l = \mathcal{A}_l(\textbf{h}_{l-1},\textbf{w}_l).
\end{equation}
This sequential dependence is called a \textit{forward locking}.
In the backward pass, the chain rule for gradient is applied:
\begin{subequations}\label{eq_backward_step}
\begin{align}
\frac{\partial \phi(\boldsymbol{\chi}^{(n)},\overline{\textbf{W}})}{\partial \textbf{w}_l} &= \frac{\partial \phi(\boldsymbol{\chi}^{(n)},\overline{\textbf{W}})}{\partial \textbf{h}_l}\frac{\partial \textbf{h}_l}{\partial \textbf{w}_l},\\
\frac{\partial \phi(\boldsymbol{\chi}^{(n)},\overline{\textbf{W}})}{\partial \textbf{h}_l} &= \frac{\partial \phi(\boldsymbol{\chi}^{(n)},\overline{\textbf{W}})}{\partial \textbf{h}_{l+1}} \frac{\partial \textbf{h}_{l+1}}{\partial \textbf{h}_l}.
\end{align}
\end{subequations}
It is obvious that the error gradients $\frac{\partial \phi(\boldsymbol{\chi}^{(n)},\overline{\textbf{W}})}{\partial \textbf{h}_l}$ need to be repeatedly propagated through the DNN from the last layer $l = L$ to the first layer $l = 1$ as the computation in layer $l$ depends on the error gradient from the layer $l+1$. This constraint causes a \textit{backward locking}. It is the main bottle neck in DNN training when $L$ is large since the required time for each step of backward pass \eqref{eq_backward_step} is significantly larger than the one of forward pass \eqref{eq_forward_step}. Additionally, updating the weight vector of the layer $l$ requires the gradient $\frac{\partial \phi(\boldsymbol{\chi}^{(n)},\overline{\textbf{W}})}{\partial \textbf{w}_l}$ which only can be found if $\frac{\partial \phi(\boldsymbol{\chi}^{(n)},\overline{\textbf{W}})}{\partial \textbf{h}_l}$ is available. This is recognized as an \textit{update locking}.
\subsection{Communication graph for multi-agent system}
In multi-agent system, it is usual to use an undirected graph $\mathcal{G} = (\mathcal{V}, \mathcal{E})$ to illustrate the communication network among $S$ agents, in which, $\mathcal{V} = \{1, 2, \dots, S\}$ is the node set and $\mathcal{E} \subset \mathcal{V} \times \mathcal{V}$ is the edge set.
Each node $i \in \mathcal{V} = \{1, 2, \dots, S\}$ corresponds to an agent while an edge $(i,j) \in \mathcal{E}$ represents that the agent $i$ can exchange information with the agent $j$.
As $\mathcal{G}$ is an undirected graph, $(i,j) \in \mathcal{E}$ implies $(j,i) \in \mathcal{E}$ for every pair of nodes $i,j \in \mathcal{V}$.
The neighbor set of the agent $i \in \mathcal{V}$ is denoted by $\mathcal{N}_i = \left\{j \in \mathcal{V}: (i,j) \in \mathcal{E}\right\}$.
The communication graph $\mathcal{G}$ is usually required to be connected. That means there exists at least one path from $i$ to $j$ for any pair $i,j \in \mathcal{V}$, i.e., there is a sequence of edges $(i_1,i_2), (i_2,i_3), \dots, (i_{s-1},i_s)$ such that $i_1 = i, i_s = j$ and $i_l \in \mathcal{V}, (i_{l-1},i_l) \in \mathcal{E}, \forall 2 \le l \le s$.
Let $\tilde{\mathcal{V}}$ be a node subset, i.e., $\tilde{\mathcal{V}} \subset \mathcal{V}$. The subgraph induced by $\tilde{\mathcal{V}}$ is defined as $\tilde{\mathcal{G}} = (\tilde{\mathcal{V}}, \tilde{\mathcal{E}})$ where $\tilde{\mathcal{E}} = \{(i,j) \in \mathcal{E}: i,j \in \tilde{\mathcal{V}}\}$.

For each communication graph, the pattern reflecting its topology can be presented by a weighted matrix $\textbf{P} = [P_{ij}] \in \mathbb{R}^{S \times S}$ as
\begin{equation}\label{eq_weight_matrix}
P_{ij} = \left\{\begin{matrix}\alpha, & (i,j) \in \mathcal{E},\\ 1 - \kappa_i\alpha, & i = j,\\ 0, &\textrm{otherwise}\end{matrix}\right.
\end{equation}
where $\kappa_i$ is the degree of the node $i$ and $\alpha \in \left(0,\frac{1}{\max_{i \in \mathcal{V}}\kappa_i}\right)$ is a given parameter.
The following lemma provides the properties of the weighted matrix $\textbf{P}$ \cite{LinXiao2004}.
\begin{lemma}
The weighted matrix $\textbf{P} \in \mathbb{R}^{S \times S}$ satisfies the following properties:
\begin{enumerate}
\item $\textbf{P}$ is symmetric and doubly stochastic, i.e., $\textbf{P} = \textbf{P}^T$, $\textbf{1}_S^T\textbf{P} = \textbf{1}_S^T$ and $\textbf{P}\textbf{1}_S = \textbf{1}_S$.
\item $\rho\left(\textbf{P} - \frac{1}{S}\textbf{1}_S\textbf{1}_S^T\right) = \gamma < 1$, where $\rho(\textbf{X})$ denotes the spectral radius of matrix $\textbf{X}$.
\end{enumerate}
\end{lemma}
\section{Distributed training}\label{sec_main}
In this section, we give the details of the proposed framework for distributed deep training. Our approach is the combination of the data parallelism and the fully decoupled parallel backpropagation algorithm.
\subsection{Data parallelism}
Let the training dataset be divided into $S$ subsets as
\[\mathcal{D} = \mathcal{D}_1 \cup \mathcal{D}_2 \cup \cdots \cup \mathcal{D}_S,\]
\[\textrm{ and } \mathcal{D}_i \cap \mathcal{D}_j = \emptyset \textrm{ if } i \neq j.\]
Denote by $\hat{\textbf{W}}_s$ the local copy of the DNN weight vector $\overline{\textbf{W}}$ corresponding to the training data subset $\mathcal{D}_s$ and the associated cost function is defined by
\[\Psi_s(\hat{\textbf{W}}_s) = \frac{1}{N}\sum\limits_{\boldsymbol{\chi}^{(n)} \in \mathcal{D}_s}\phi(\boldsymbol{\chi}^{(n)},\hat{\textbf{W}}_s).\]
Then the training problem \eqref{eq_training_problem} can be reformulated into the following equivalent problem:
\begin{subequations}\label{eq_dl_problem}
\begin{align}
\min\limits_{\hat{\textbf{w}}_{s}, \forall s}& \sum\limits_{s = 1}^{S} \Psi_s(\hat{\textbf{W}}_s)\\
\textrm{s.t. }& \hat{\textbf{W}}_{1} = \hat{\textbf{W}}_{2} = \cdots = \hat{\textbf{W}}_{S}.
\end{align}
\end{subequations}
In the data parallelism, the optimal solution of the problem \eqref{eq_dl_problem} is found by the cooperation of $S$ workers. Assuming the communication among these workers is represented by a connected and undirected graph $\mathcal{G} = (\mathcal{V}, \mathcal{E})$ where the node set is $\mathcal{V} = \{1, 2, \dots, S\}$ and the edge set is $\mathcal{E} \subsetneq \mathcal{V} \times \mathcal{V}$. Each node $s\in\mathcal{V}$ corresponds to one worker. The subset $\mathcal{D}_s$ and the local copy $\hat{\textbf{W}}_s$ are assigned to the worker $s$. Additionally, let $\textbf{P} = [P_{sr}] \in \mathbb{R}^{S \times S}$ be the weighted matrix corresponding to the graph $\mathcal{G}$ (as defined by \eqref{eq_weight_matrix}).
The following update rule can be implemented for each worker $s \in \mathcal{V}$ to cooperatively solve the problem \eqref{eq_dl_problem}:
\begin{equation}\label{eq_consensus_update}
\hat{\textbf{W}}_s(t+1) = \sum\limits_{r \in \mathcal{N}_s}P_{sr}\left(\hat{\textbf{W}}_r(t) - \eta_t\nabla\Phi_r(t)\right)
\end{equation}
where $\Phi_s(t) = \frac{|\mathcal{D}_s|}{BN}\sum\limits_{\boldsymbol{\chi}^{(n)} \in \mathcal{B}_s(t)}\phi(\boldsymbol{\chi}^{(n)},\hat{\textbf{W}}_s)$ is the cost function corresponding to the mini-batch $\mathcal{B}_s(t)$ sampled from the training data subset $\mathcal{D}_s$ at time $t \ge 0$. The mini-batch $\mathcal{B}_s(t)$ is chosen to have the cardinality $|\mathcal{B}_s(t)| = B$ for all $s \in \mathcal{V}$ and $t \ge 0$.
Data parallelism focuses on distributing the original training data set across multiple workers to alleviate the computation load required in centralized setting. However, each worker still copes with locking issues when computing the gradient $\nabla\Phi_s(t)$. To overcome this difficulty, the technique to decouple the backpropagation algorithm needs to be employed.
\subsection{Fully decoupled parallel backpropagation algorithm}\label{subsec_stalegradient}
There are some versions of decoupled parallel backpropagation technique to enhance the parallel computation in deep learning. In this part, we review the algorithm proposed in \cite{HuipingZhuang2022} and applied to the stochastic gradient method \eqref{eq_SGD}.
\begin{figure*}[htb]
\begin{center}
\centerline{\includegraphics[width=1.8\columnwidth]{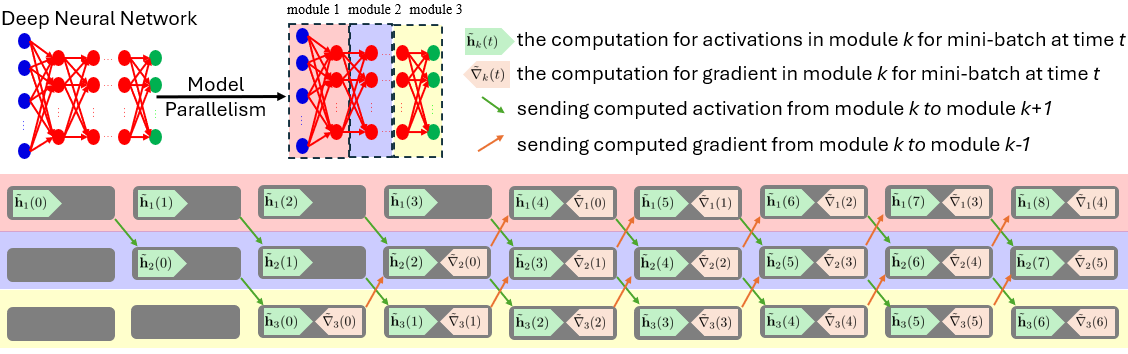}}
\caption{Training DNN using stale gradient with three modules.}
\label{fig_stale_gradient}
\vspace{5pt}
\centerline{\includegraphics[width=1.8\columnwidth]{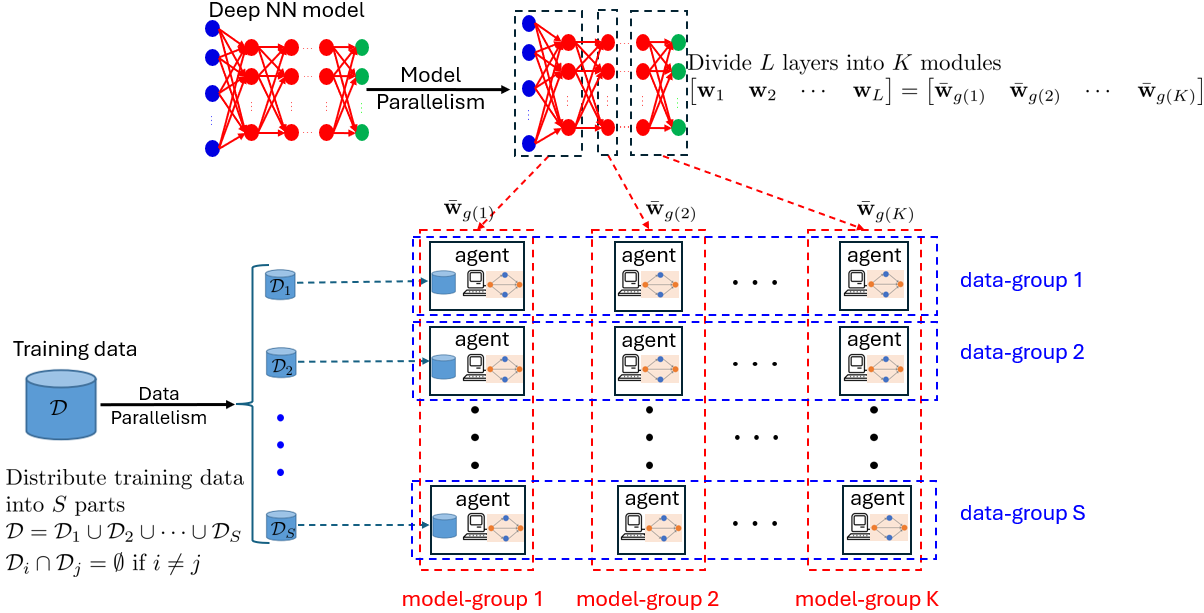}}
\caption{Proposed framework.}
\label{fig_proposed_framework}
\end{center}
\end{figure*}
The key features of this fully decoupled parallel backpropagation algorithm consist of the continuous operation for minimizing idle time and the use of stale gradient in updating DNN weight.

Let the layer indexes $\{1, 2, \dots, L\}$ be split into $K$ groups as $\{g(1), g(2), \dots, g(K)\}$ where $g(k) = \{p_k, p_k+1, \dots, q_{k}\}, p_1 = 0, p_{k} < q_k = p_{k+1}-1$, and $q_K = L$.
Denote by $\bar{\textbf{w}}_{g(k)}$ the weights corresponding to the layers in the group $g(k)$. We have
\[\overline{\textbf{W}} = \textrm{col}\{\bar{\textbf{w}}_{g(1)}, \bar{\textbf{w}}_{g(2)}, \cdots, \bar{\textbf{w}}_{g(K)}\}\]
During the training process, the computation for the update of $\bar{\textbf{w}}_{g(k)}$ is assigned to one module $k$.
In each iteration, the module $k$ performs forward pass and backward pass corresponding to two different mini-batches. The activation of the last layer $q_k$ and the gradient of the first layer $p_k$, which have been computed in this iteration, will be be sent respectively to the module $k+1$ and the module $k-1$.
\begin{itemize}
\item \textbf{Forward}: At each iteration $t \ge 0$, module $1$ samples a mini-batch $\mathcal{B}(t) \subset \mathcal{D}$ and module $k \ge 2$ receives the activation $\textbf{h}_{q_{k-1}}(t-k+1)$ from the module $k-1$ if $t \ge k - 1$. After that, each module $k$ computes the activation $\textbf{h}_l^{(k)}(t-k+1)$ from $l = p_k$ to $q_k$; then sends $\textbf{h}_{q_k}^{(k)}(t-k+1)$ to the module $k+1$ if $k < K$. Here, we use the notation $\textbf{h}_l^{(k)}(\tau)$, if $\tau \ge 0$, to denote the stacked vector of activation values computed by the module $k$ for the layer $l$ corresponding to the mini-batch sampled at the time $\tau$ with the weight vector $\textbf{w}_l(\tau+k-1)$. That means $\textbf{h}_0^{(k)}(\tau) = \textrm{col}\{\boldsymbol{\chi}^{(n)}: \boldsymbol{\chi}^{(n)} \in \mathcal{B}(\tau)\}$, $\textbf{h}_{p_k}^{(k)}(\tau) = \mathcal{A}_{p_k}(\textbf{h}_{q_{k-1}}^{(k-1)}(\tau),\textbf{w}_{p_k}(\tau+k-1))$ and $\textbf{h}_l^{(k)}(\tau) = \mathcal{A}_l(\textbf{h}_{l-1}^{(k)}(\tau),\textbf{w}_l(\tau+k-1)) \forall p_k+1 \le l \le q_k$. In addition, let $\tilde{\textbf{h}}_k(\tau) = \textrm{col}\{\textbf{h}_{p_k}^{(k)}(\tau), \textbf{h}_{p_k+1}^{(k)}(\tau), \cdots, \textbf{h}_{q_k}^{(k)}(\tau)\}$.
\item \textbf{Backward}: With the forward process described as above, the last module $K$ can only compute its activation corresponding to the mini-batch $\mathcal{B}(t)$ at the time $t+K-1$, then it can compute its corresponding gradients at this time. Define $\tilde{\textbf{W}}(\tau) = \textrm{col}\{\bar{\textbf{w}}_{g(1)}(\tau), \bar{\textbf{w}}_{g(2)}(\tau+1), \cdots, \bar{\textbf{w}}_{g(K)}(\tau+K-1)\}$, we have $\textbf{h}_L^{(K)}(\tau) = \mathcal{A}(\textbf{h}_0^{(1)}(\tau),\tilde{\textbf{W}}(\tau))$.
As the gradients of the module $k$ depends on the one of the first layer in the module $k+1$, at iteration $t$, the module $k$ has the available information to compute its error gradients associated with the mini-batch sampled at the time $t-2K+k+1$. The necessary information includes $\frac{\partial \phi(\boldsymbol{\chi}^{(n)},\tilde{\textbf{W}}(t-2K+k+1))}{\partial \textbf{h}_{q_k+1}}, \forall \boldsymbol{\chi}^{(n)} \in \mathcal{B}(t-2K+k+1)$, received from the module $k+1$ (if $k < K$) and $\textbf{w}_{l}(t-2K+2k)), \forall p_k \le l \le q_k$, belonging to itself. So the module $k$ can compute the error gradient $\frac{\partial \phi(\boldsymbol{\chi}^{(n)},\tilde{\textbf{W}}(t-2K+k+1))}{\partial \textbf{h}_{l}}, \forall \boldsymbol{\chi}^{(n)} \in \mathcal{B}(t-2K+k+1),$ for layers from $l = q_k$ to $p_k$.
\end{itemize}
The computation process of the fully decoupled parallel backpropagation algorithm is illustrated in Fig. \ref{fig_stale_gradient} where the DNN is split into three modules.
The update of the weights for each module is based on the stale gradient as follows.
\begin{equation}
\bar{\textbf{w}}_{g(k)}(t+1) = \bar{\textbf{w}}_{g(k)}(t) - \tilde{\nabla}_{g(k)}\Phi(t-2K+k+1)
\end{equation}
where $\tilde{\nabla}_{g(k)}\Phi(\tau) = \sum\limits_{\boldsymbol{\chi}^{(n)} \in \mathcal{B}(\tau)}\frac{\partial \phi(\boldsymbol{\chi}^{(n)},\tilde{\textbf{W}}(\tau))}{\partial \bar{\textbf{w}}_{g(k)}}$ if $\tau \ge 0$ and $\tilde{\nabla}_{g(k)}\Phi(\tau) = \textbf{0}$ otherwise.
\subsection{Proposed framework}
In this paper, we propose a distributed deep learning by employing the decoupled parallel  backpropagation algorithm into the gradient computation of the update \eqref{eq_consensus_update}.
Assume there are $SK$ agents cooperating to train the DNN. These agents are group into $S$ data-groups and $K$ model-groups. Each data-group $s, \forall 1 \le s \le S,$ consists of $K$ agents working on the training data subset $\mathcal{D}_s$ while each model-group $k, \forall 1 \le k \le K$, includes $S$ agents assigned for computation in the module $k$. Fig. \ref{fig_proposed_framework} illustrates our proposed framework.
We use $(s,k)$ to denote the index of each agent representing its membership in data-group $s$ and model-group $k$. Let $\hat{\textbf{w}}_{s,k}$ be the local weight vector of the agent $(s,k)$. This is a copy of the weight vector of the module $k$, i.e., $\bar{\textbf{w}}_{g(k)}$. In addition, we have
\begin{equation}
\hat{\textbf{W}}_s = \left[\begin{matrix}\hat{\textbf{w}}_{s,1}^T & \hat{\textbf{w}}_{s,2}^T & \cdots & \hat{\textbf{w}}_{s,K}^T\end{matrix}\right]^T, \forall 1 \le s \le S.
\end{equation}
The constraint (\ref{eq_dl_problem}b) can be rewritten as
\begin{equation}
\hat{\textbf{w}}_{1,k} = \hat{\textbf{w}}_{2,k} = \cdots = \hat{\textbf{w}}_{S,k}, \forall 1 \le k \le K.
\end{equation}
The task of sampling a mini-batch training data $\mathcal{B}_s(t) \subsetneq \mathcal{D}_s$ belongs to the agent $(s,1)$ for all $t \ge 0$. Employing the decoupled parallel algorithm, in every iteration $t \ge 0$, each agent $(s,k)$ computes activation vectors $\textbf{h}_l^{(s,k)}(t-k+1)$ for all layers $l = p_k, \dots, q_k$ corresponding to the mini-batch $\mathcal{B}_s(t-k+1)$ and computes the error gradient $\frac{\partial \phi(\boldsymbol{\chi}^{(n)},\hat{\textbf{W}}_s(t-2K+k+1))}{\partial \hat{\textbf{w}}_{s,k}}$ for every sample $\boldsymbol{\chi}^{(n)} \in \mathcal{B}_s(t-2K+k+1)$.

Let $\mathcal{G}^{comm} = (\mathcal{V}^{comm}, \mathcal{E}^{comm})$ be the communication graph of the multiagent system where $\mathcal{V}^{comm} = \{(1,1), (1,2), \dots, (1,K), \dots, (S,1), (S,2), \dots, (S,K)\}$ and $\mathcal{E}^{comm} \subsetneq \mathcal{V}^{comm} \times \mathcal{V}^{comm}$.
Denote by $\mathcal{V}_s^D$ and $\mathcal{V}_k^M$ the node sets corresponding to the data-group $s$ and model-group $k$, respectively. Accordingly, we define $\mathcal{G}_s^D$ and $\mathcal{G}_k^M$ as the sub-graphs induced by the node sets $\mathcal{V}_s^D, \forall 1 \le s \le S$, and $\mathcal{V}_k^M, \forall 1 \le k \le K$, from the communication graph $\mathcal{G}^{comm}$, respectively.
Because agents in every data-group $s$ need to compute activation values and gradients of $K$ modules, their computations are consecutively interdependent. So $\mathcal{G}_s^D$ must be a line.
As the agents in each model-group $k$ estimate the same weights for a module $k$, it is required that $\mathcal{G}_k^M$ must be a connected graph to allow the agreement for estimated weight vectors of agents in $\mathcal{V}_k^M$.
\begin{assumption}\label{aspt_topology}
The communication graph satisfies that
\begin{enumerate}
\item $\mathcal{G}_s^D$ is an undirected line, $\forall 1 \le s \le S$,
\item $\mathcal{G}_k^M$ is undirected and connected, $\forall 1 \le k \le K$.
\end{enumerate}
\end{assumption}
In the remainder of this paper, we assume that all sub-graph $\mathcal{G}_k^M, \forall 1 \le k \le K$, have the same topology as the connected graph $\mathcal{G}$ (mentioned in Subsection 3.1) for easy formulation.
Let $\mathcal{N}_{s,k}$ be the set of neighboring agents of the agent $(s,k)$ in the module-group $k$, i.e., $\mathcal{N}_{s,k} = \{(r,k) \in \mathcal{V}_k^M: ((s,k), (r,k)) \in \mathcal{E}^{comm}\}$.
Using stale gradient in \eqref{eq_consensus_update}, the update for each agent $(s,k)$ is given as follows.
\begin{subequations}\label{eq_dl_update}
\begin{align}
&\hat{\textbf{u}}_{s,k}(t) = \hat{\textbf{w}}_{s,k}(t) - \eta_t\hat{\nabla}_{\hat{\textbf{w}}_{s,k}}\Phi_s(t-2K+k+1)\\
&\hat{\textbf{w}}_{s,k}(t+1) = \sum\limits_{(r,k) \in \mathcal{N}_{s,k}}P_{sr}\hat{\textbf{u}}_{r,k}(t)
\end{align}
\end{subequations}
where $\hat{\textbf{u}}_{s,k}$ is a virtual variable vector, $\Phi_s(\tau) = \frac{|\mathcal{D}_s|}{BN} \sum\limits_{\boldsymbol{\chi}^{(n)} \in \mathcal{B}_s(\tau)}\phi(\boldsymbol{\chi}^{(n)},\tilde{\textbf{W}}_s(\tau))$ is the local cost function corresponding to the mini-batch $\mathcal{B}_s(\tau)$, and $\hat{\nabla}_{\hat{\textbf{w}}_{s,k}}\Phi_s(\tau) = \frac{|\mathcal{D}_s|}{BN}\sum\limits_{\boldsymbol{\chi}^{(n)} \in \mathcal{B}_s(\tau)}\frac{\partial \phi(\boldsymbol{\chi}^{(n)},\tilde{\textbf{W}}_s(\tau))}{\partial \hat{\textbf{w}}_{s,k}}$ if $\tau \ge 0$ and $\hat{\nabla}_{\hat{\textbf{w}}_{s,k}}\Phi_s(\tau) = \textbf{0}$ otherwise. Similar to $\tilde{\textbf{W}}(\tau)$, we define $\tilde{\textbf{W}}_s(\tau)$ as
\[\tilde{\textbf{W}}_s(\tau) = \textrm{col}\{\hat{\textbf{w}}_{s,k}(\tau+k-1): 1 \le k \le K\}.\]
We summarize the proposed distributed training method in Algorithm \ref{alg_proposed} (presented in Appendix-D).
\section{Convergence analysis}
This section analyzes the convergence of Algorithm \ref{alg_proposed} under some commonly encountered assumptions in deep training.
\begin{assumption}\label{aspt_lipschitz}
The gradient of the loss function $\phi\left(\boldsymbol{\chi}^{(n)}, \overline{\textbf{W}}\right)$ is Lipschitz continuous with Lipschitz constant $\varrho > 0$ for all data point $\boldsymbol{\chi}^{(n)} \in \mathcal{D}$. That means, for every $\boldsymbol{\chi}^{(n)} \in \mathcal{D}$, we have $\left|\left|\nabla\phi\left(\boldsymbol{\chi}^{(n)}, \overline{\textbf{W}}^{(1)}\right) - \nabla\phi\left(\boldsymbol{\chi}^{(n)}, \overline{\textbf{W}}^{(2)}\right)\right|\right| \le \varrho\left|\left|\overline{\textbf{W}}^{(1)} - \overline{\textbf{W}}^{(2)}\right|\right|$ for all $\overline{\textbf{W}}^{(1)}, \overline{\textbf{W}}^{(2)} \in \mathbb{R}^d$.
\end{assumption}
\begin{assumption}\label{aspt_unnoised}
Each agent $(s,1)$ samples mini-batches satisfying
\[\mathbb{E}\left[\nabla_{\hat{\textbf{W}}_s}\Phi_s(t)\right] = \nabla\Psi_s(\hat{\textbf{W}}_s(t)), \forall t \ge 0\]
\end{assumption}
To bound the variance of the stochastic gradient, we assume the second moment of the stochastic gradient is upper bounded.
\begin{assumption}\label{aspt_secondmoment}
There exists constant $\sigma > 0$ such that
\begin{equation}
\left|\left|\nabla \phi\left(\boldsymbol{\xi}^{(n)}, \overline{\textbf{W}}\right)\right|\right|^2 \le \sigma^2, \forall \overline{\textbf{W}} \in \mathbb{R}^d,
\end{equation}
for every data point $\boldsymbol{\xi}^{(n)} \in \mathcal{D}$.
\end{assumption}
Let $\hat{\textbf{W}}$ and $\tilde{\textbf{W}}^{avr}(t)$ be the staked vector and the average of estimated weight vectors of all data-groups, respectively. We have
\[\hat{\textbf{W}} = \left[\begin{matrix}\hat{\textbf{W}}_1^T & \hat{\textbf{W}}_2^T & \cdots & \hat{\textbf{W}}_S^T\end{matrix}\right]^T\]
\[\tilde{\textbf{W}}^{avr}(t) = \frac{1}{S}\sum_{s = 1}^{S}\hat{\textbf{W}}_s(t) = \frac{1}{S}\left(\textbf{1}_S^T \otimes \textbf{I}_S\right) \hat{\textbf{W}}(t)\]
Let $\boldsymbol{\delta}(t) = \hat{\textbf{W}}(t) - \left(\textbf{1}_S \otimes \textbf{I}_d\right) \tilde{\textbf{W}}^{avr}(t)$ be the consensus error at the time $t \ge 0$.
The following lemma provides the bound for the consensus error:
\begin{lemma}\label{lm_consensus}
Under Assumption \ref{aspt_topology} and Assumption \ref{aspt_secondmoment}, it holds for all $t \ge 0$
\begin{equation}\label{eq_main_convergence_error}
\left|\left|\boldsymbol{\delta}(t+1)\right|\right| \le \gamma^{t+1}\left|\left|\boldsymbol{\delta}(0)\right|\right| + \sigma\sqrt{\frac{K}{BS}}\sum\limits_{\tau = 0}^{t}\gamma^{t+1-\tau}\eta_{\tau}
\end{equation} 
\end{lemma}
\begin{proof}
See Appendix-A.
\end{proof}
\subsection{Fixed stepsize}
We first analyze the convergence of our proposed distributed training method \eqref{eq_dl_update} when the step-size is fixed, i.e., $\eta_t = \eta, \forall t \ge 0$.
\begin{theorem}\label{th_main_1}
Assume Assumption \ref{aspt_topology}, \ref{aspt_lipschitz}, \ref{aspt_unnoised} and \ref{aspt_secondmoment} hold and the fixed stepsize satisfies $\eta \le \frac{S}{\varrho}$. We have
\begin{equation}\label{eq_theorem1_1}
\left|\left|\boldsymbol{\delta}(t+1)\right|\right| \le \left|\left|\boldsymbol{\delta}(0)\right|\right|\gamma^{t+1} + \frac{\gamma\eta}{1-\gamma}
\end{equation}
\begin{align}
\frac{1}{T}\sum_{t = 0}^{T-1}\mathbb{E}\left[\left|\left|\nabla\Psi\left(\tilde{\textbf{W}}^{avr}(t)\right)\right|\right|^2\right] \le\textrm{ }& \frac{2S\left(\Psi^{(0)} - \Psi^*\right)}{\eta T}\nonumber\\ &+ \frac{M_1}{T} + M_2\eta\label{eq_theorem1_2}
\end{align}
where $\gamma = \rho\left(\textbf{P}-\frac{1}{S}\textbf{1}_S\textbf{1}_S^T\right) < 1$, $\Psi^{(0)} = \Psi\left(\tilde{\textbf{W}}^{avr}(0)\right)$, $\Psi^*$ is the optimal value of the training problem \eqref{eq_training_problem}, $M_1 = \left(24K^3+1\right)\left(S + \varrho\eta\right)\varrho^2\left(\frac{\gamma^2}{1 - \gamma^2}\left|\left|\boldsymbol{\delta}(0)\right|\right|^2 + \frac{\gamma}{1 - \gamma}\left|\left|\boldsymbol{\delta}(0)\right|\right|\eta\right)$ and $M_2 = \frac{\varrho K \sigma^2}{BS^2} + (\eta S+\varrho\eta^2)\frac{6K^3\varrho^2\sigma^2}{BS^4} + \left(24K^3+1\right)\left(S + \varrho\eta\right)\varrho^2\left(\frac{\gamma}{1 - \gamma}\right)^2\eta^2$.
\end{theorem}
\begin{proof}
See Appendix-B.
\end{proof}
Theorem \ref{th_main_1} shows that both the norm of the consensus error $\boldsymbol{\delta}(t)$ and the average norm of the gradients of total cost functions over $T$ iterations are upper bound by quantities depending on the step-size $\eta$. We can reduce the step-size to obtain the solution more closer to the critical points. Nevertheless, the drawback of this selection is the decrease of the convergence speed.
\subsection{Diminishing step-size}
To guarantee surely the convergence to critical points, our proposed training method requires diminishing step-sizes, which satisfy the following assumption.
\begin{assumption}\label{aspt_dimish}
Step-sizes have the following properties:
\[\eta_t\frac{\varrho}{S} \le 1, \eta_t > \eta_{t+1}, \forall t \ge 0\]
\[\lim\limits_{T \rightarrow \infty}\sum_{t = 0}^{T-1}\eta_t = \infty\]
\[\lim\limits_{T \rightarrow \infty}\sum_{t = 0}^{T-1}\eta_t^2 = M_0 < \infty\]
\end{assumption}
It can be verified that $\eta_t = \frac{\eta^*}{t+1}$, where $\eta^* \le \frac{S}{\varrho}$, is an example satisfying Assumption \ref{aspt_dimish}.
\begin{theorem}\label{th_main_2}
Assume Assumption \ref{aspt_topology}, \ref{aspt_lipschitz}, \ref{aspt_unnoised}, \ref{aspt_secondmoment} and \ref{aspt_dimish} hold. We have
\begin{equation}\label{eq_theorem2_1}
\lim\limits_{t \rightarrow \infty}\left|\left|\boldsymbol{\delta}(t)\right|\right| = 0
\end{equation}
Let $\tilde{\textbf{W}}^{avr}(\tau)$ be chosen randomly from $\{\tilde{\textbf{W}}^{avr}(t): 0 \le t \le T-1\}$ with probabilities proportional to $\{\eta_0, \eta_1, \dots, \eta_{T-1}\}$. Then, we can obtain
\begin{equation}\label{eq_theorem2_2}
\lim\limits_{\tau \rightarrow \infty}\mathbb{E}\left[\nabla\Psi\left(\tilde{\textbf{W}}^{avr}(\tau)\right)\right] = 0
\end{equation}
\end{theorem}
\begin{proof}
See Appendix-C.
\end{proof}
\begin{figure*}[htb]
\begin{center}
\includegraphics[width=0.65\columnwidth]{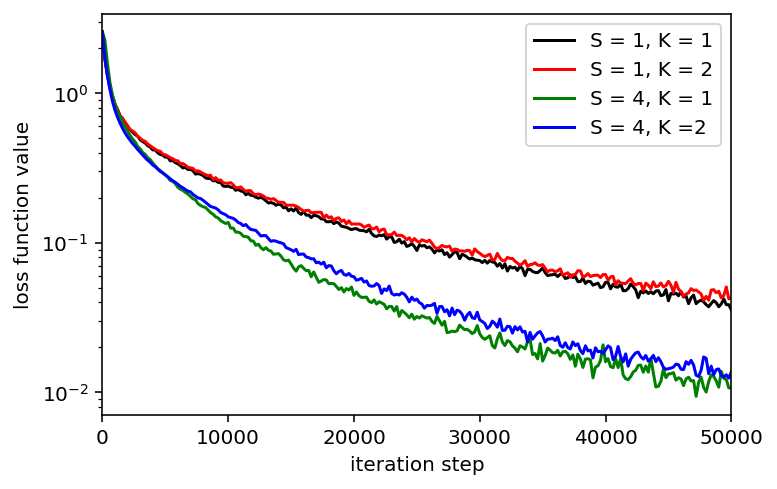}
\includegraphics[width=0.65\columnwidth]{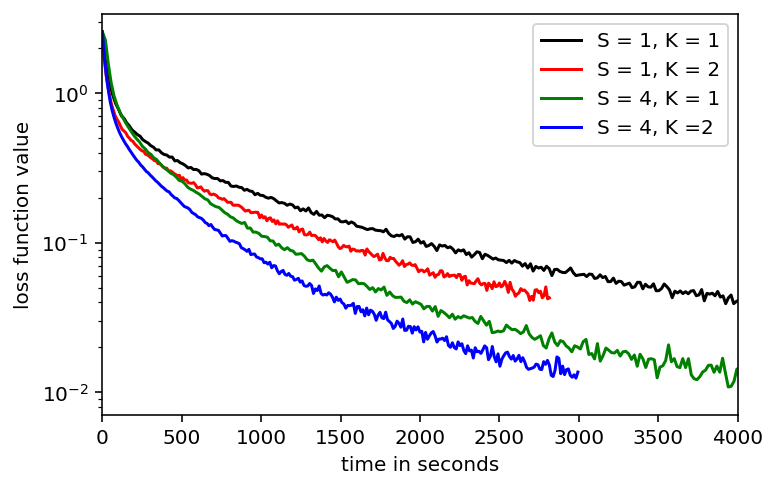}
\includegraphics[width=0.65\columnwidth]{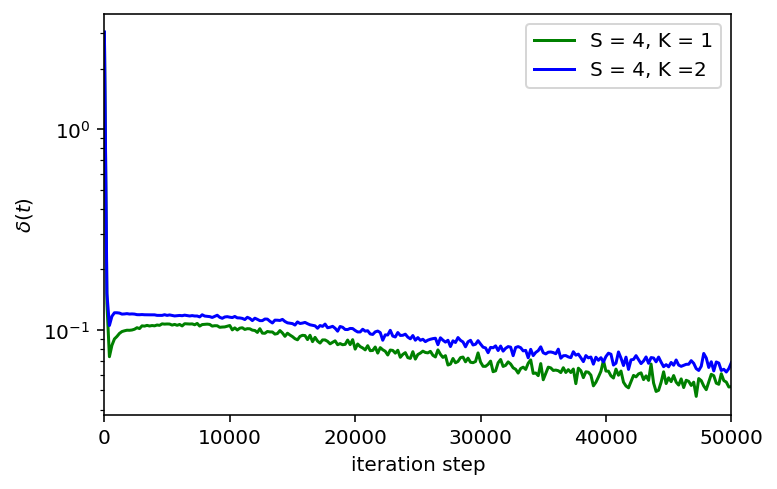}
\caption{Training performance results under Strategy I.}
\label{fig_result_1}
\end{center}
\end{figure*}
\begin{figure*}[htb]
\begin{center}
\includegraphics[width=0.65\columnwidth]{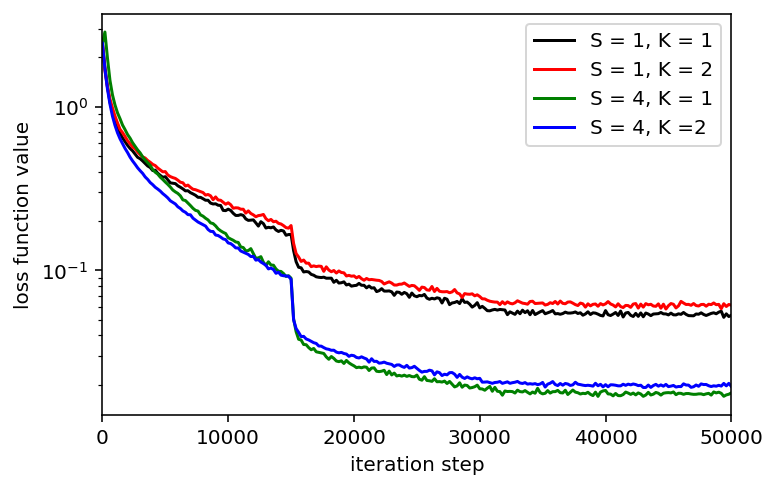}
\includegraphics[width=0.65\columnwidth]{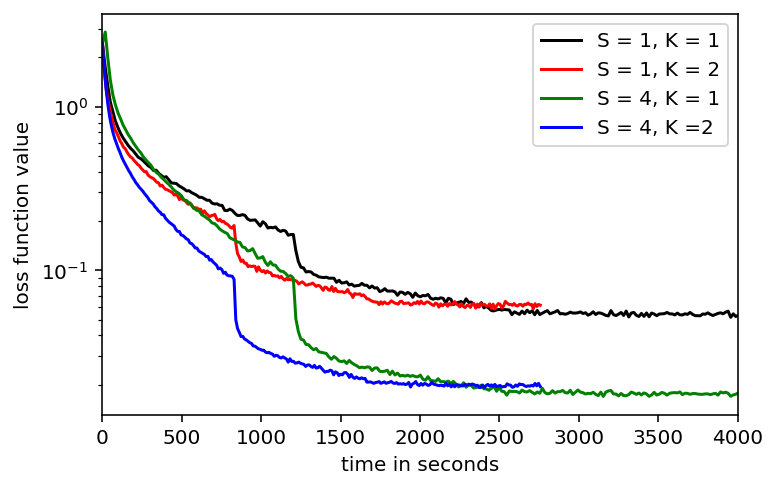}
\includegraphics[width=0.65\columnwidth]{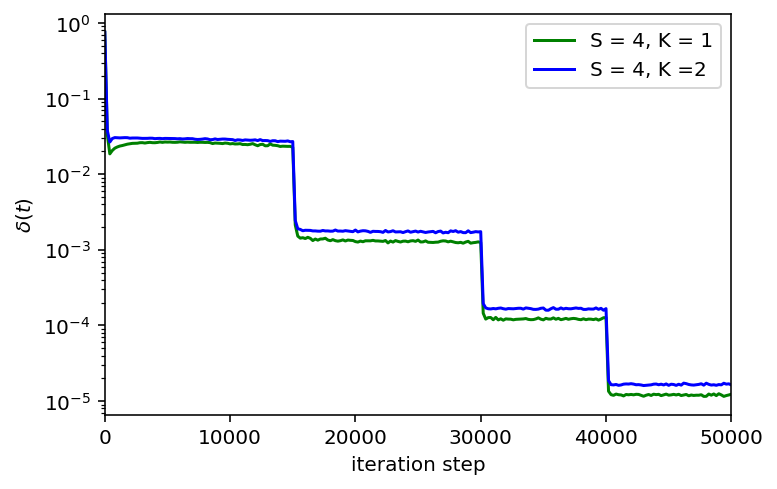}
\caption{Training performance results under Strategy II.}
\label{fig_result_2}
\end{center}
\end{figure*}
\section{Experiments}\label{sec_exprm}
To evaluate the effectiveness of the proposed distributed training method, we conduct experiments on training a neural network (ResNet-20) for classification tasks. These experiments were performed on a GPU (NVIDIA GeForce GTX 1060 3Gbs) with the CIFAR-10 training dataset, which consists of $50,000$ samples of $32 \times 32$ color images spanning ten classes.
We compare the training performance across the following four training methods:
\begin{enumerate}
\item Centralized method ($S = 1, K = 1$): This represents the stochastic gradient method \eqref{eq_SGD} using the traditional backpropagation algorithm. The training dataset is centralized on a single computer.
\item Decoupled model method ($S = 1, K = 2$): This method employs the fully decoupled parallel backpropagation algorithm with the training dataset centralized on a single computer. The network is split into two submodels, and the computations for each submodel are assigned to a separate agent.
\item Data parallel method ($S = 4, K = 1$): The update law is defined by the equation \eqref{eq_consensus_update}, where gradients are computed using the traditional backpropagation algorithm. The training dataset is distributed among four agents.
\item Distributed training method ($S = 4, K = 2$): This is our proposed training method. The network is divided into two-model groups, and the training dataset is distributed among four data groups.
\end{enumerate}
For each method, we perform $50,000$ iteration steps. In each step, the mini-batch size is set to $194$.
The training process is conducted under two strategies of step-size selections:
\begin{align}
\textrm{\textbf{Strategy I: }}& \eta_t = 0.1, \forall t \ge 0.\label{eq_stepsize_st1}\\
\textrm{\textbf{Strategy II: }}& \eta_t = \left\{\begin{matrix}
0.1 & \textrm{if } t \le 15000\\
0.01 & \textrm{if } 15000 < t \le 30000\\
0.001 & \textrm{if } 30000 < t \le 40000\\
0.0001 & \textrm{if } t > 40000.
\end{matrix}\right..\label{eq_stepsize_st2}
\end{align}

Fig. \ref{fig_result_1} shows the training performances under constant step-size \eqref{eq_stepsize_st1} and Fig. \ref{fig_result_2} corresponds to the Strategy II (i.e., the equation \eqref{eq_stepsize_st2}).
In these figures, the black, red, green and blue curves represent the performance results of the centralized method, decoupled model method, data parallel method and distributed training method, respectively.
In Fig. \ref{fig_result_1} and Fig. \ref{fig_result_2}, figures in their first columns describe the evolution of the loss function values over training iteration steps. Due to using stale gradient, the performance result of the decoupled model method is little worse than the centralized method. This drawback is solved by increasing the amount of training data processed in each iteration as the approach of the distributed training method. Though the data parallel method has better result than the distributed training method when using the same number of iterations, the training time is substantially longer. To process one mini-batch training data, the methods using traditional backpropagation algorithm need $85$ ms while the ones using fully decoupled parallel backpropagation algorithm need $58$ ms. Figures in second columns (of Fig. \ref{fig_result_1} and Fig. \ref{fig_result_2}) describe the evolution of the loss function values over training time. The distributed training method provides the best performance result in four compared training methods with the same training time.

Let $\textbf{w}_{s,l}(t)$ be the vectorization of the weight in the layer $l, \forall 1 \le l \le L$, corresponding to the vector $\hat{\textbf{W}}_s(t)$. For the data parallel method and the distributed training method, we measure the disagreement of the estimated weight vectors $\hat{\textbf{W}}_s(t), \forall 1 \le s \le S$, with their average by the following equation:
\begin{equation}\label{eq_max_consensus_error}
\delta(t) = \max\limits_{\substack{1 \le l \le L,\\ 1 \le s \le S}}\left|\left|\textbf{w}_{s,l}(t) - \frac{1}{S}\sum\limits_{s = 1}^{S}\textbf{w}_{s,l}(t)\right|\right|.
\end{equation}
Figures in third columns (of Fig. \ref{fig_result_1} and Fig. \ref{fig_result_2}) represent the evolution of $\delta(t)$ (corresponding to the data parallel method and the distributed training method) regarding training iterations. We can observe that, in both methods, $\delta(t)$ reduce significantly fast to values less than the chosen step-sizes.
\section{Conclusion}\label{sec_conclude}
This paper introduces a distributed method for accelerating the training of deep neural networks by integrating data parallelism with the fully decoupled parallel backpropagation algorithm. The proposed approach is fully decentralized, as it eliminates the need for central coordination, relying solely on local communications between neighboring agents. The effectiveness of the method is established both theoretically and empirically. Under commonly encountered conditions in deep learning, the method is proven to achieve critical points at a sublinear convergence rate. Experimental results on classification tasks further demonstrate that the proposed method significantly accelerates the training process without compromising accuracy.
\bibliography{mylib}
\bibliographystyle{icml2025}
\newpage
\appendix
\onecolumn
\allowdisplaybreaks
\section{Proof of Lemma \ref{lm_consensus}}
Define $\hat{\nabla}\Upsilon(t)$ as the stacked vector of all gradients used in iteration $t \ge 0$. That means
\begin{align*}
\hat{\nabla}\Upsilon(t) = \left[\begin{matrix} \hat{\nabla}\Phi_1(t) \\ \hat{\nabla}\Phi_2(t)\\ \vdots\\ \hat{\nabla}\Phi_S(t)\end{matrix}\right] \textrm{ where }
\hat{\nabla}\Phi_s(t) = \left[\begin{matrix} \hat{\nabla}_{\hat{\textbf{w}}_{s,1}}\Phi_s(t-2K+2) \\ \hat{\nabla}_{\hat{\textbf{w}}_{s,2}}\Phi_s(t-2K+3)\\ \vdots\\ \hat{\nabla}_{\hat{\textbf{w}}_{s,K}}\Phi_s(t-K+1)\end{matrix}\right]
\end{align*}
From the definition, we have $\left|\left|\hat{\nabla}_{\hat{\textbf{w}}_{s,k}}\Phi_s(\tau)\right|\right|^2 \le \frac{\sigma^2}{BS^2}$ for all $s, k$ and $\tau \ge 0$. So, for all $\tau \ge 0$, it is guaranteed that
\begin{equation}
\left|\left|\hat{\nabla}\Phi_s(\tau)\right|\right|^2 \le \frac{K\sigma^2}{BS^2}, \forall 1 \le s \le S, \textrm{ and } \left|\left|\hat{\nabla}\Upsilon(\tau)\right|\right|^2 \le \frac{K\sigma^2}{BS}.
\end{equation}

According to the update \eqref{eq_dl_update}, we have
\begin{equation}
\hat{\textbf{W}}(t+1) = \left(\textbf{P} \otimes \textbf{I}_d\right)\left(\hat{\textbf{W}}(t) - \eta_t\hat{\nabla}\Upsilon(t)\right)
\end{equation}

Consider the consensus error $\boldsymbol{\delta}(t+1)$, we have
\begin{align}
\boldsymbol{\delta}(t+1) &= \hat{\textbf{W}}(t+1) - \left(\textbf{1}_S \otimes \textbf{I}_d\right) \tilde{\textbf{W}}^{avr}(t+1)\nonumber\\
&= \left(\left(\textbf{I}_S - \frac{1}{S}\textbf{1}_S\textbf{1}_S^T\right) \otimes \textbf{I}_S\right)\hat{\textbf{W}}(t+1)\nonumber\\
&= \left(\left(\textbf{I}_S - \frac{1}{S}\textbf{1}_S\textbf{1}_S^T\right) \otimes \textbf{I}_d\right) \left(\textbf{P} \otimes \textbf{I}_d\right) \left(\hat{\textbf{W}}(t) - \eta_t\hat{\nabla}\Upsilon(t)\right)\label{eq_matrixcomp_0}
\end{align}
Consider that
\begin{align}
\left(\left(\textbf{I}_S - \frac{1}{S}\textbf{1}_S\textbf{1}_S^T\right) \otimes \textbf{I}_d\right) \left(\textbf{P} \otimes \textbf{I}_d\right) &= \left(\left(\textbf{P} - \frac{1}{S}\textbf{1}_S\textbf{1}_S^T\textbf{P}\right) \otimes \left(\textbf{I}_d\textbf{I}_d\right)\right)\nonumber\\
&= \left(\left(\textbf{P} - \frac{1}{S}\textbf{1}_S\textbf{1}_S^T\right) \otimes \textbf{I}_d\right)\label{eq_matrixcomp_1}
\end{align}
In addition, we have
\begin{align}
\left(\left(\textbf{P} - \frac{1}{S}\textbf{1}_S\textbf{1}_S^T\right) \otimes \textbf{I}_d\right)\left(\left(\textbf{I}_S - \frac{1}{S}\textbf{1}_S\textbf{1}_S^T\right) \otimes \textbf{I}_d\right) &= \left(\left(\textbf{P} - \frac{1}{S}\textbf{1}_S\textbf{1}_S^T\right)\left(\textbf{I}_S - \frac{1}{S}\textbf{1}_S\textbf{1}_S^T\right)\right) \otimes \textbf{I}_d\nonumber\\
&= \left(\textbf{P} - \frac{1}{S}\textbf{1}_S\textbf{1}_S^T  - \frac{1}{S}\textbf{P}\textbf{1}_S\textbf{1}_S^T + \frac{1}{S^2}\textbf{1}_S\textbf{1}_S^T\textbf{1}_S\textbf{1}_S^T\right) \otimes \textbf{I}_d\nonumber\\
&= \left(\textbf{P} - \frac{2}{S}\textbf{1}_S\textbf{1}_S^T + \frac{1}{S^2}\textbf{1}_SS\textbf{1}_S^T\right) \otimes \textbf{I}_d\nonumber\\
&= \left(\textbf{P} - \frac{1}{S}\textbf{1}_S\textbf{1}_S^T\right) \otimes \textbf{I}_d\label{eq_matrixcomp_2}
\end{align}
Define $\boldsymbol{\Gamma} = \left(\textbf{P} - \frac{1}{S}\textbf{1}_S\textbf{1}_S^T\right) \otimes \textbf{I}_d$. From \eqref{eq_matrixcomp_1}, \eqref{eq_matrixcomp_2} and \eqref{eq_matrixcomp_0}, we have
\begin{align*}
\boldsymbol{\delta}(t+1) &= \boldsymbol{\Gamma}\boldsymbol{\delta}(t) + \eta_t\boldsymbol{\Gamma}\hat{\nabla}\Upsilon(t)\\
&= \boldsymbol{\Gamma}^2\boldsymbol{\delta}(t-1) + \eta_{t-1}\boldsymbol{\Gamma}^2\hat{\nabla}\Upsilon(t-1) + \eta_t\boldsymbol{\Gamma}\hat{\nabla}\Upsilon(t)\\
&= \cdots\\
&= \boldsymbol{\Gamma}^{t+1}\boldsymbol{\delta}(0) + \sum\limits_{\tau = 0}^{t}\eta_{\tau}\boldsymbol{\Gamma}^{t+1-\tau}\hat{\nabla}\Upsilon(\tau)
\end{align*}
Taking the norm of the above equation, we obtain
\begin{align}
\left|\left|\boldsymbol{\delta}(t+1)\right|\right| &\le \left|\left|\boldsymbol{\Gamma}^{t+1}\boldsymbol{\delta}(0)\right|\right| + \left|\left|\sum\limits_{\tau = 0}^{t}\eta_{\tau}\boldsymbol{\Gamma}^{t+1-\tau}\hat{\nabla}\Upsilon(\tau)\right|\right| \le \gamma^{t+1}\left|\left|\boldsymbol{\delta}(0)\right|\right| + \sigma\sqrt{\frac{K}{BS}}\sum\limits_{\tau = 0}^{t}\gamma^{t+1-\tau}\eta_{\tau}
\end{align}
because $\left|\left|\hat{\nabla}\Upsilon(\tau)\right|\right| \le \sigma\sqrt{\frac{K}{BS}}$ for all $\tau \ge 0$.
\section{Proof of Theorem \ref{th_main_1}}
Define $\Lambda(t) = \sum\limits_{\tau = 0}^{t}\gamma^{t+1-\tau}\eta_{\tau}$.

For $\eta(t) = \eta, \forall t \ge 0$, we have 
\[\Lambda(t) = \eta\sum\limits_{\tau = 0}^{t}\gamma^{t+1-\tau} \le \eta\sum\limits_{\tau = 1}^{\infty}\gamma^{\tau}.\]
It is well-known that $\sum\limits_{\tau = 1}^{\infty}\gamma^{\tau} = \frac{\gamma}{1 - \gamma}$.
Thus we have
\begin{align}
\left|\left|\textbf{W}(t+1) - \left(\textbf{1}_S \otimes \textbf{I}_d\right) \tilde{\textbf{W}}^{avr}(t+1)\right|\right| \le \gamma^{t+1}\left|\left|\boldsymbol{\delta}(0)\right|\right| + \frac{\gamma}{1 - \gamma}\eta.
\end{align}

Since the loss function $\phi(\cdot)$ is $\varrho$ Lipschitz continuous, we have $\Psi(\overline{\textbf{W}})$ is $\varrho$ Lipschitz continuous. 
This implies
\begin{align*}
\Psi(\tilde{\textbf{W}}^{avr}(t+1)) \le\textrm{ }& \Psi(\tilde{\textbf{W}}^{avr}(t)) + \left(\nabla\Psi(\tilde{\textbf{W}}^{avr}(t))\right)^T\left(\tilde{\textbf{W}}^{avr}(t+1) - \tilde{\textbf{W}}^{avr}(t)\right) + \frac{\varrho}{2}\left|\left|\tilde{\textbf{W}}^{avr}(t+1) - \tilde{\textbf{W}}^{avr}(t)\right|\right|^2\\
=\textrm{ }& \Psi(\tilde{\textbf{W}}^{avr}(t)) - \tilde{\eta}_t\left(\nabla\Psi(\tilde{\textbf{W}}^{avr}(t))\right)^T\left(\sum_{s = 1}^{S}\hat{\nabla}\Phi_s(t)\right) + \frac{\varrho\tilde{\eta}_t^2}{2}\left|\left|\sum_{s = 1}^{S}\hat{\nabla}\Phi_s(t)\right|\right|^2,
\end{align*}
where $\tilde{\eta}_t = \frac{\eta_t}{S}$.
Taking expectation for both sides of the above equation, we obtain
\begin{align}
\mathbb{E}\left[\Psi(\tilde{\textbf{W}}^{avr}(t+1))\right] \le \mathbb{E}\left[\Psi(\tilde{\textbf{W}}^{avr}(t))\right] - \tilde{\eta}_t\left(\nabla\Psi(\tilde{\textbf{W}}^{avr}(t))\right)^T\mathbb{E}\left[\sum_{s = 1}^{S}\hat{\nabla}\Phi_s(t)\right] + \frac{\varrho\tilde{\eta}_t^2}{2}\mathbb{E}\left[\left|\left|\sum_{s = 1}^{S}\hat{\nabla}\Phi_s(t)\right|\right|^2\right]\label{eq_proofconv_temp1}
\end{align}
Define $\hat{\nabla}\Psi_s(t) = \left[\begin{matrix} \nabla_{\hat{\textbf{w}}_{s,1}}\Psi_s(\hat{\textbf{W}}_1(t-2K+2)) \\ \nabla_{\hat{\textbf{w}}_{s,2}}\Psi_s(\hat{\textbf{W}}_2(t-2K+4))\\ \vdots\\ \nabla_{\hat{\textbf{w}}_{s,K}}\Psi_s(\hat{\textbf{W}}_K(t))\end{matrix}\right]$.
According to Assumption \ref{aspt_unnoised}, we have $\mathbb{E}\left[\hat{\nabla}\Phi_s(t)\right] = \hat{\nabla}\Psi_s(t)$.
The equation \eqref{eq_proofconv_temp1} becomes
\begin{align*}
\mathbb{E}\left[\Psi(\tilde{\textbf{W}}^{avr}(t+1))\right] \le\textrm{ }& \mathbb{E}\left[\Psi(\tilde{\textbf{W}}^{avr}(t))\right] - \tilde{\eta}_t\left(\nabla\Psi(\tilde{\textbf{W}}^{avr}(t))\right)^T\left(\sum_{s = 1}^{S}\hat{\nabla}\Psi_s(t)\right) + \frac{\varrho\tilde{\eta}_t^2}{2}\mathbb{E}\left[\left|\left|\sum_{s = 1}^{S}\hat{\nabla}\Phi_s(t)\right|\right|^2\right]\\
=\textrm{ }& \mathbb{E}\left[\Psi(\tilde{\textbf{W}}^{avr}(t))\right] - \tilde{\eta}_t\left(\nabla\Psi(\tilde{\textbf{W}}^{avr}(t))\right)^T\left(\sum_{s = 1}^{S}\hat{\nabla}\Psi_s(t) - \nabla\Psi(\tilde{\textbf{W}}^{avr}(t)) + \nabla\Psi(\tilde{\textbf{W}}^{avr}(t))\right)\\ & + \frac{\varrho\tilde{\eta}_t^2}{2}\mathbb{E}\left[\left|\left|\sum_{s = 1}^{S}\hat{\nabla}\Phi_s(t) - \nabla\Psi(\tilde{\textbf{W}}^{avr}(t)) + \nabla\Psi(\tilde{\textbf{W}}^{avr}(t))\right|\right|^2\right]\\
=\textrm{ }& \mathbb{E}\left[\Psi(\tilde{\textbf{W}}^{avr}(t))\right] - \left(\tilde{\eta}_t - \frac{\varrho\tilde{\eta}_t^2}{2}\right)\left|\left|\nabla\Psi(\tilde{\textbf{W}}^{avr}(t))\right|\right|^2 + \frac{\varrho\tilde{\eta}_t^2}{2}\left|\left|\sum_{s = 1}^{S}\hat{\nabla}\Phi_s(t) - \nabla\Psi(\tilde{\textbf{W}}^{avr}(t))\right|\right|^2\\ &+ \left(\varrho\tilde{\eta}_t^2 - \tilde{\eta}_t\right)\left(\nabla\Psi(\tilde{\textbf{W}}^{avr}(t))\right)^T\left(\sum_{s = 1}^{S}\hat{\nabla}\Psi_s(t) - \nabla\Psi(\tilde{\textbf{W}}^{avr}(t))\right)
\end{align*}
By choosing $\eta_t \le \frac{S}{\varrho}$, we have $\varrho\tilde{\eta}_t^2 - \tilde{\eta}_t < 0$.
In addition, $\pm\left(\nabla\Psi(\tilde{\textbf{W}}^{avr}(t))\right)^T\left(\sum_{s = 1}^{S}\hat{\nabla}\Psi_s(t) - \nabla\Psi(\tilde{\textbf{W}}^{avr}(t))\right) \le \frac{1}{2}\left|\left|\nabla\Psi(\tilde{\textbf{W}}^{avr}(t))\right|\right|^2 + \frac{1}{2}\left|\left|\sum_{s = 1}^{S}\hat{\nabla}\Psi_s(t) - \nabla\Psi(\tilde{\textbf{W}}^{avr}(t))\right|\right|^2$. So, we have
\begin{align}
\mathbb{E}\left[\Psi(\tilde{\textbf{W}}^{avr}(t+1))\right] \le\textrm{ }& \mathbb{E}\left[\Psi(\tilde{\textbf{W}}^{avr}(t))\right] - \left(\tilde{\eta}_t - \frac{\varrho\tilde{\eta}_t^2}{2}\right)\left|\left|\nabla\Psi(\tilde{\textbf{W}}^{avr}(t))\right|\right|^2 + \frac{\varrho\tilde{\eta}_t^2}{2}\mathbb{E}\left[\left|\left|\sum_{s = 1}^{S}\hat{\nabla}\Phi_s(t) - \nabla\Psi(\tilde{\textbf{W}}^{avr}(t))\right|\right|^2\right]\nonumber\\ &+ \left(\tilde{\eta}_t - \varrho\tilde{\eta}_t^2\right)\left(\frac{1}{2}\left|\left|\nabla\Psi(\tilde{\textbf{W}}^{avr}(t))\right|\right|^2 + \frac{1}{2}\left|\left|\sum_{s = 1}^{S}\hat{\nabla}\Psi_s(t) - \nabla\Psi(\tilde{\textbf{W}}^{avr}(t))\right|\right|^2\right)\nonumber\\
=\textrm{ }& \mathbb{E}\left[\Psi(\tilde{\textbf{W}}^{avr}(t))\right] - \frac{\tilde{\eta}_t}{2}\left|\left|\nabla\Psi(\tilde{\textbf{W}}^{avr}(t))\right|\right|^2 + \frac{\varrho\tilde{\eta}_t^2}{2}\mathbb{E}\left[\left|\left|\sum_{s = 1}^{S}\hat{\nabla}\Phi_s(t) - \nabla\Psi(\tilde{\textbf{W}}^{avr}(t))\right|\right|^2\right]\nonumber\\&+ \frac{\tilde{\eta}_t - \varrho\tilde{\eta}_t^2}{2}\left|\left|\sum_{s = 1}^{S}\hat{\nabla}\Psi_s(t) - \nabla\Psi(\tilde{\textbf{W}}^{avr}(t))\right|\right|^2\label{eq_proofconv_temp2}
\end{align}
Consider that $\sum_{s = 1}^{S}\hat{\nabla}\Phi_s(t) - \nabla\Psi(\tilde{\textbf{W}}^{avr}(t)) = \sum_{s = 1}^{S}\hat{\nabla}\Phi_s(t) - \sum_{s = 1}^{S}\hat{\nabla}\Psi_s(t) + \sum_{s = 1}^{S}\hat{\nabla}\Psi_s(t) - \nabla\Psi(\tilde{\textbf{W}}^{avr}(t))$.
According to Cauchy-Schwarz inequality, we have
\begin{align*}
\mathbb{E}\left[\left|\left|\sum_{s = 1}^{S}\hat{\nabla}\Phi_s(t) - \nabla\Psi(\tilde{\textbf{W}}^{avr}(t))\right|\right|^2\right] \le 2\mathbb{E}\left[\left|\left|\sum_{s = 1}^{S}\hat{\nabla}\Phi_s(t) - \sum_{s = 1}^{S}\hat{\nabla}\Psi_s(t)\right|\right|^2\right] + 2\left|\left|\sum_{s = 1}^{S}\hat{\nabla}\Psi_s(t) - \nabla\Psi(\tilde{\textbf{W}}^{avr}(t))\right|\right|^2
\end{align*}
Since $\mathbb{E}\left[\sum_{s = 1}^{S}\hat{\nabla}\Phi_s(t)\right] = \sum_{s = 1}^{S}\hat{\nabla}\Psi_s(t)$, we have $\mathbb{E}\left[\left|\left|\sum_{s = 1}^{S}\hat{\nabla}\Phi_s(t) - \sum_{s = 1}^{S}\hat{\nabla}\Psi_s(t)\right|\right|^2\right] = \mathbb{E}\left[\left|\left|\sum_{s = 1}^{S}\hat{\nabla}\Phi_s(t)\right|\right|^2\right] - \sum_{s = 1}^{S}\left|\left|\hat{\nabla}\Psi_s(t)\right|\right|^2 \le \mathbb{E}\left[\left|\left|\sum_{s = 1}^{S}\hat{\nabla}\Phi_s(t)\right|\right|^2\right] \le \frac{K\sigma^2}{BS}$.
This implies
\begin{equation}\label{eq_proofconv_temp3}
\mathbb{E}\left[\left|\left|\sum_{s = 1}^{S}\hat{\nabla}\Phi_s(t) - \nabla\Psi(\tilde{\textbf{W}}^{avr}(t))\right|\right|^2\right] \le 2\frac{K\sigma^2}{BS} + 2\left|\left|\sum_{s = 1}^{S}\hat{\nabla}\Psi_s(t) - \nabla\Psi(\tilde{\textbf{W}}^{avr}(t))\right|\right|^2
\end{equation}
From \eqref{eq_proofconv_temp2} and \eqref{eq_proofconv_temp3}, we have
\begin{align}
\mathbb{E}\left[\Psi(\tilde{\textbf{W}}^{avr}(t+1))\right] \le\textrm{ }& \mathbb{E}\left[\Psi(\tilde{\textbf{W}}^{avr}(t))\right] - \frac{\tilde{\eta}_t}{2}\left|\left|\nabla\Psi(\tilde{\textbf{W}}^{avr}(t))\right|\right|^2 + \frac{\varrho K \sigma^2}{BS}\tilde{\eta}_t^2 + \frac{\tilde{\eta}_t + \varrho\tilde{\eta}_t^2}{2}\left|\left|\sum_{s = 1}^{S}\hat{\nabla}\Psi_s(t) - \nabla\Psi(\tilde{\textbf{W}}^{avr}(t))\right|\right|^2\label{eq_proofconv_temp4}
\end{align}

Next, we find the bound of $\left|\left|\sum_{s = 1}^{S}\hat{\nabla}\Psi_s(t) - \nabla\Psi(\tilde{\textbf{W}}^{avr}(t))\right|\right|^2$ as follows.
\begin{align}
\left|\left|\sum_{s = 1}^{S}\hat{\nabla}\Psi_s(t) - \nabla\Psi(\tilde{\textbf{W}}^{avr}(t))\right|\right|^2 \le\textrm{ }& 2\left|\left|\sum_{s = 1}^{S}\hat{\nabla}\Psi_s(t) - \sum_{s = 1}^{S}\nabla\Psi_s(\hat{\textbf{W}}_s(t))\right|\right|^2 + 2\left|\left|\sum_{s = 1}^{S}\nabla\Psi_s(\hat{\textbf{W}}_s(t)) - \nabla\Psi(\tilde{\textbf{W}}^{avr}(t))\right|\right|^2\label{eq_proofconv_temp5}
\end{align}
In the following analysis, we consider the fixed step-size $\eta_t = \eta, \forall t \ge 0$. We have $\tilde{\eta}_t = \frac{\eta}{S}$.
\begin{align}
\left|\left|\sum_{s = 1}^{S}\hat{\nabla}\Psi_s(t) - \sum_{s = 1}^{S}\nabla\Psi_s(\hat{\textbf{W}}_s(t))\right|\right|^2 \le\textrm{ }& S\sum_{s = 1}^{S}\left|\left|\hat{\nabla}\Psi_s(t) - \nabla\Psi_s(\hat{\textbf{W}}_s(t))\right|\right|^2\nonumber\\
=\textrm{ }& S\sum_{s = 1}^{S}\sum_{k = 1}^{K}\left|\left|\nabla_{\hat{\textbf{w}}_{s,k}}\Psi_s(\hat{\textbf{W}}_s(t-2K+2k)) - \nabla_{\hat{\textbf{w}}_{s,k}}\Psi_s(\hat{\textbf{W}}_s(t))\right|\right|^2\nonumber\\
\le\textrm{ }& S\sum_{s = 1}^{S}\sum_{k = 1}^{K}\left|\left|\nabla\Psi_s(\hat{\textbf{W}}_s(t-2K+2k)) - \nabla\Psi_s(\hat{\textbf{W}}_s(t))\right|\right|^2\nonumber\\
\le\textrm{ }& S\varrho^2\sum_{s = 1}^{S}\sum_{k = 1}^{K}\left|\left|\hat{\textbf{W}}_s(t-2K+2k) - \hat{\textbf{W}}_s(t)\right|\right|^2\nonumber\\
=\textrm{ }& S\varrho^2\sum_{s = 1}^{S}\sum_{k = 1}^{K}\left|\left|\sum_{\tau = \max\{0,t-2K+2k\}}^{t-1}\left(\hat{\textbf{W}}_s(\tau+1) - \hat{\textbf{W}}_s(\tau)\right)\right|\right|^2\nonumber\\
=\textrm{ }& S\varrho^2\sum_{k = 1}^{K}\left|\left|\sum_{\tau = \max\{0,t-2K+2k\}}^{t-1}\left(\hat{\textbf{W}}(\tau+1) - \hat{\textbf{W}}(\tau)\right)\right|\right|^2\nonumber\\
\le\textrm{ }& S\varrho^2\sum_{k = 1}^{K}\left\{\left(t-1-\Delta_{t,k}\right)\sum_{\tau = \Delta_{t,k}}^{t-1}\left|\left|\hat{\textbf{W}}(\tau+1) - \hat{\textbf{W}}(\tau)\right|\right|^2\right\}\label{eq_proofconv_temp6}
\end{align}
where $\Delta_{t,k} = \max\{0,t-2K+2k\}$. It is clear that $t-1-\Delta_{t,k} < 2K$.
In addition, we have
\begin{align}
\left|\left|\hat{\textbf{W}}(\tau+1) - \hat{\textbf{W}}(\tau)\right|\right|^2 =\textrm{ }& \left|\left|\hat{\textbf{W}}(\tau+1) - \left(\textbf{1}_S \otimes \textbf{I}_d\right) \tilde{\textbf{W}}^{avr}(\tau+1) + \left(\textbf{1}_S \otimes \textbf{I}_d\right) \tilde{\textbf{W}}^{avr}(\tau+1) - \left(\textbf{1}_S \otimes \textbf{I}_d\right) \tilde{\textbf{W}}^{avr}(\tau)\right.\right.\nonumber\\
&\left.\left. + \left(\textbf{1}_S \otimes \textbf{I}_d\right) \tilde{\textbf{W}}^{avr}(\tau) - \hat{\textbf{W}}(\tau)\right|\right|^2\nonumber\\
=\textrm{ }& \left|\left|\boldsymbol{\delta}(\tau+1) + \left(\textbf{1}_S \otimes \textbf{I}_d\right) \tilde{\textbf{W}}^{avr}(\tau+1) - \left(\textbf{1}_S \otimes \textbf{I}_d\right) \tilde{\textbf{W}}^{avr}(\tau) + \boldsymbol{\delta}(\tau)\right|\right|^2\nonumber\\
\le\textrm{ }& \left|\left|\boldsymbol{\delta}(\tau+1) + \textbf{1}_S \otimes \sum_{s=1}^{S}\frac{\eta}{S}\hat{\nabla}\Phi_s(\tau) + \boldsymbol{\delta}(\tau)\right|\right|^2\nonumber\\
\le\textrm{ }& 3\left|\left|\boldsymbol{\delta}(\tau+1)\right|\right|^2 + 3\left|\left|\textbf{1}_S \otimes \sum_{s=1}^{S}\frac{\eta}{S}\hat{\nabla}\Phi_s(\tau)\right|\right|^2 + 3\left|\left|\boldsymbol{\delta}(\tau)\right|\right|^2\nonumber\\
\le\textrm{ }& 3\left|\left|\boldsymbol{\delta}(\tau+1)\right|\right|^2 + 3\left|\left|\boldsymbol{\delta}(\tau)\right|\right|^2 + 3\left|\left|\eta\sum_{s=1}^{S}\hat{\nabla}\Phi_s(\tau)\right|\right|^2\\
\le\textrm{ }& 3\left|\left|\boldsymbol{\delta}(\tau+1)\right|\right|^2 + 3\left|\left|\boldsymbol{\delta}(\tau)\right|\right|^2 + 3\frac{K\sigma^2}{BS^2}\eta^2\label{eq_proofconv_temp7}
\end{align}
From \eqref{eq_proofconv_temp6} and \eqref{eq_proofconv_temp7}, we have
\begin{align}
\left|\left|\sum_{s = 1}^{S}\hat{\nabla}\Psi_s(t) - \sum_{s = 1}^{S}\nabla\Psi_s(\hat{\textbf{W}}_s(t))\right|\right|^2 &\le 2KS\varrho^2\left(3\frac{K^2\sigma^2}{BS^2}\eta^2 + 3\sum_{k = 1}^{K}\sum_{\tau = \Delta_{t,k}}^{t-1}\left(\left|\left|\boldsymbol{\delta}(\tau+1)\right|\right|^2 + \left|\left|\boldsymbol{\delta}(\tau)\right|\right|^2\right)\right)\nonumber\\
&= \frac{6K^3\varrho^2\sigma^2}{BS}\eta^2 + 6KS\varrho^2\sum_{k = 1}^{K}\sum_{\tau = \Delta_{t,k}}^{t-1}\left(\left|\left|\boldsymbol{\delta}(\tau+1)\right|\right|^2 + \left|\left|\boldsymbol{\delta}(\tau)\right|\right|^2\right)\label{eq_proofconv_temp8}
\end{align}
Consider that 
\begin{align}
\left|\left|\sum_{s = 1}^{S}\nabla\Psi_s(\hat{\textbf{W}}_s(t)) - \nabla\Psi(\tilde{\textbf{W}}^{avr}(t))\right|\right|^2 =\textrm{ }& \left|\left|\sum_{s = 1}^{S}\left(\nabla\Psi_s(\hat{\textbf{W}}_s(t)) - \nabla\Psi_s(\tilde{\textbf{W}}^{avr}(t))\right)\right|\right|^2\nonumber\\
\le\textrm{ }& S\sum_{s = 1}^{S} \left|\left|\nabla\Psi_s(\hat{\textbf{W}}_s(t)) - \nabla\Psi_s(\tilde{\textbf{W}}^{avr}(t))\right|\right|^2\nonumber\\
\le\textrm{ }& \frac{S\varrho^2(\max_s\{|\mathcal{D}_s|\})^2}{N^2}\sum_{s = 1}^{S} \left|\left|\hat{\textbf{W}}_s(t) - \tilde{\textbf{W}}^{avr}(t)\right|\right|^2\nonumber\\
=\textrm{ }& \frac{S\varrho^2\max_s\{|\mathcal{D}_s|^2\}}{N^2} \left|\left|\boldsymbol{\delta}(t)\right|\right|^2\label{eq_proofconv_temp9}
\end{align}
From \eqref{eq_proofconv_temp4}, \eqref{eq_proofconv_temp5}, \eqref{eq_proofconv_temp8} and \eqref{eq_proofconv_temp9}, we have
\begin{align}
\mathbb{E}\left[\Psi(\tilde{\textbf{W}}^{avr}(t+1))\right] \le\textrm{ }& \mathbb{E}\left[\Psi(\tilde{\textbf{W}}^{avr}(t))\right] - \frac{\eta}{2S}\left|\left|\nabla\Psi(\tilde{\textbf{W}}^{avr}(t))\right|\right|^2 + \left(\frac{\varrho K \sigma^2}{BS} + (\eta S+\varrho\eta^2)\frac{6K^3\varrho^2\sigma^2}{BS^3}\right)\frac{\eta_t^2}{S^2}\nonumber\\
&+ \frac{\eta(S+\varrho\eta)}{S^2}\left(6KS\varrho^2\sum_{k = 1}^{K}\sum_{\tau = \Delta_{t,k}}^{t-1}\left(\left|\left|\boldsymbol{\delta}(\tau+1)\right|\right|^2 + \left|\left|\boldsymbol{\delta}(\tau)\right|\right|^2\right) + \frac{S\varrho^2\max_s\{|\mathcal{D}_s|^2\}}{N^2} \left|\left|\boldsymbol{\delta}(t)\right|\right|^2\right)\label{eq_main_convergence_1}
\end{align}
Summing both sides of \eqref{eq_main_convergence_1} from $t = 0$ to $T-1$, we obtain
\begin{align*}
\mathbb{E}\left[\Psi\left(\tilde{\textbf{W}}^{avr}(T)\right)\right] - \Psi^{(0)} \le\textrm{ }& -\frac{\eta}{2S}\sum_{t = 0}^{T-1}\left|\left|\nabla\Psi\left(\tilde{\textbf{W}}^{avr}(t)\right)\right|\right|^2 + \left(\frac{\varrho K \sigma^2}{BS^3} + (\eta S +\varrho\eta^2)\frac{6K^3\varrho^2\sigma^2}{BS^5}\right)\eta^2T + \frac{(S+\varrho\eta)}{S^2}\Xi(T)\eta
\end{align*}
where $\Xi(T) = \sum_{t = 0}^{T-1}\left(6KS\varrho^2\sum_{k = 1}^{K}\sum_{\tau = \Delta_{t,k}}^{t-1}\left(\left|\left|\boldsymbol{\delta}(\tau+1)\right|\right|^2 + \left|\left|\boldsymbol{\delta}(\tau)\right|\right|^2\right) + \frac{S\varrho^2\max_s\{|\mathcal{D}_s|^2\}}{N^2} \left|\left|\boldsymbol{\delta}(t)\right|\right|^2\right)$.
Multiplying both sides of the above equation with $\frac{2S}{\eta T}$, we have
\begin{equation}\label{eq_main_convergence_2}
\frac{1}{T}\sum_{t = 0}^{T-1}\mathbb{E}\left[\left|\left|\nabla\Psi\left(\tilde{\textbf{W}}^{avr}(t)\right)\right|\right|^2\right] \le 2S\frac{\Psi^{(0)} - \mathbb{E}\left[\Psi\left(\tilde{\textbf{W}}^{avr}(T)\right)\right]}{\eta T} + 2\left(\frac{\varrho K \sigma^2}{BS^2} + (\eta S + \varrho\eta^2)\frac{6K^3\varrho^2\sigma^2}{BS^4}\right)\eta + \frac{S+\varrho\eta}{TS}\Xi(T)
\end{equation}
Let $\tilde{\varrho} = \frac{\varrho S\max_s\{|\mathcal{D}_s|\}}{N} \le S\varrho$.
Consider $\Xi(T)$.
\begin{align}
\Xi(T) &= 6KS\varrho^2\sum_{t = 0}^{T-1}\sum_{k = 1}^{K}\sum_{\tau = \Delta_{t,k}}^{t-1}\left(\left|\left|\boldsymbol{\delta}(\tau+1)\right|\right|^2 + \left|\left|\boldsymbol{\delta}(\tau)\right|\right|^2\right) + \frac{\tilde{\varrho}^2}{S}\sum_{t = 0}^{T-1}\left|\left|\boldsymbol{\delta}(t)\right|\right|^2\nonumber\\
&= 6KS\varrho^2\sum_{k = 1}^{K}\sum_{\tau = \Delta_{T,k}}^{T-1}\sum_{t = 0}^{\tau-1}\left(\left|\left|\boldsymbol{\delta}(t+1)\right|\right|^2 + \left|\left|\boldsymbol{\delta}(t)\right|\right|^2\right) + \frac{\tilde{\varrho}^2}{S}\sum_{t = 0}^{T-1}\left|\left|\boldsymbol{\delta}(t)\right|\right|^2\nonumber\\
&\le (6KS\varrho^2)K(2K)2\sum_{t = 0}^{T}\left|\left|\boldsymbol{\delta}(t)\right|\right|^2 + S \varrho^2\sum_{t = 0}^{T-1}\left|\left|\boldsymbol{\delta}(t)\right|\right|^2\nonumber\\
&\le \left(24K^3S+S\right)\varrho^2\sum_{t = 0}^{T}\left|\left|\boldsymbol{\delta}(t)\right|\right|^2\label{eq_proofconv_temp10}
\end{align}
Because $\left|\left|\boldsymbol{\delta}(t)\right|\right| \le \gamma^{t}\left|\left|\boldsymbol{\delta}(0)\right|\right| + \frac{\gamma}{1 - \gamma}\eta$, we have
\begin{align}
\sum_{t = 0}^{T}\left|\left|\boldsymbol{\delta}(t)\right|\right|^2 &= \left|\left|\boldsymbol{\delta}(0)\right|\right|^2\sum_{t = 0}^{T}\gamma^{2t} + \eta\left|\left|\boldsymbol{\delta}(0)\right|\right|\sum_{t = 0}^{T}\gamma^{t} + T\left(\frac{\gamma}{1 - \gamma}\right)^2\eta^2\nonumber\\
&\le \frac{\gamma^2}{1 - \gamma^2}\left|\left|\boldsymbol{\delta}(0)\right|\right|^2 + \frac{\gamma}{1 - \gamma}\left|\left|\boldsymbol{\delta}(0)\right|\right|\eta + T\left(\frac{\gamma}{1 - \gamma}\right)^2\eta^2\label{eq_proofconv_temp11}
\end{align}
Substituting \eqref{eq_proofconv_temp11} into \eqref{eq_proofconv_temp10}, then substituting \eqref{eq_proofconv_temp10} into \eqref{eq_main_convergence_2}, we have
\begin{equation}\label{eq_proofconv_temp12}
\frac{1}{T}\sum_{t = 0}^{T-1}\mathbb{E}\left[\left|\left|\nabla\Psi\left(\tilde{\textbf{W}}^{avr}(t)\right)\right|\right|^2\right] \le \frac{2S\left(\Psi^{(0)} - \mathbb{E}\left[\Psi\left(\tilde{\textbf{W}}^{avr}(T)\right)\right)\right]}{\eta T} + \frac{M_1}{T} + M_2\eta
\end{equation}
where 
\begin{align*}
M_1 &= \left(24K^3+1\right)\left(S + \varrho\eta\right)\varrho^2\left(\frac{\gamma^2}{1 - \gamma^2}\left|\left|\boldsymbol{\delta}(0)\right|\right|^2 + \frac{\gamma}{1 - \gamma}\left|\left|\boldsymbol{\delta}(0)\right|\right|\eta\right)\\
M_2 &= \frac{\varrho K \sigma^2}{BS^2} + (\eta S+\varrho\eta^2)\frac{6K^3\varrho^2\sigma^2}{BS^4} + \left(24K^3+1\right)\left(S + \varrho\eta\right)\varrho^2\left(\frac{\gamma}{1 - \gamma}\right)^2\eta^2\\
\end{align*}
Because $\mathbb{E}\left[\Psi\left(\tilde{\textbf{W}}^{avr}(T)\right)\right] \ge \Psi^*$, \eqref{eq_proofconv_temp12} implies \eqref{eq_theorem1_2}.
\section{Proof of Theorem \ref{th_main_2}}
From the definition of $\Lambda(t)$, we have
\begin{align*}
\Lambda(\tau+1) = \sum\limits_{l = 0}^{\tau+1}\gamma^{\tau+2-l}\eta_l = \gamma\sum\limits_{l = 0}^{\tau+1}\gamma^{\tau+1-l}\eta_l = \gamma\sum\limits_{l = 0}^{\tau}\gamma^{\tau+1-l}\eta_l + \gamma\eta_{\tau+1} = \gamma\Lambda(\tau) + \gamma\eta_{\tau+1}
\end{align*}
Iterating recursion of the above equation is
\begin{align*}
\Lambda(\tau+\varsigma) = \gamma^{\varsigma}\Lambda(\tau) + \gamma^{\varsigma}\eta_{\tau} + \gamma^{\varsigma-1}\eta_{\tau+1} + \cdots + \gamma\eta_{\tau+\varsigma-1}
\end{align*}
Because $\eta_{l} \ge \eta_{l+1}$ for all $l \ge 0$, we have
\[\Lambda(\tau+\varsigma) \le \gamma^{\varsigma}\Lambda(\tau) + \eta_{\tau}\sum_{l=1}^{\varsigma}\]
According to the Cauchy–Schwarz inequality, for all $\tau \ge 0$, we have
\begin{align*}
\Lambda(\tau) = \sum\limits_{l = 0}^{\tau}\gamma^{\tau+1-l}\eta_l &\le \left(\sum\limits_{l = 0}^{\tau}\gamma^{2\tau+2-2l}\right)\left(\sum\limits_{l = 0}^{\tau}\eta_l^2\right)\\
&\le \frac{\gamma^2}{1 - \gamma^2}M_0
\end{align*}
For any $\epsilon > 0$, there exists $\tau^*, \varsigma^* > 0$ such that
\[\left|\left|\boldsymbol{\delta}(0)\right|\right|\gamma^{\tau^*+\varsigma^*} \le \frac{\epsilon}{3},\]
\[\gamma^{\varsigma^*}\Lambda(\tau^*) \le \frac{\epsilon}{3},\]
\[\eta_{\tau^*} \le \frac{\epsilon}{3}\frac{1-\gamma}{\gamma}\]
From \eqref{eq_main_convergence_error}, we have
\begin{equation}
\left|\left|\boldsymbol{\delta}(t)\right|\right| \le \left|\left|\boldsymbol{\delta}(0)\right|\right|\gamma^{\tau^*+\varsigma^*} + \gamma^{\varsigma^*}\Lambda(\tau^*) + \eta_{\tau^*}\sum_{l = 1}^{t-\tau^*}\gamma^l \le \frac{\epsilon}{3} + \frac{\epsilon}{3} + \eta_{\tau^*}\frac{\gamma}{1-\gamma} = \epsilon
\end{equation}
for all $t \ge \tau^*+\varsigma^*$.
Thus $\lim\limits_{t \rightarrow \infty}\left|\left|\boldsymbol{\delta}(t)\right|\right| = 0$.

By similar analysis in the previous section, we have
\begin{align}
\mathbb{E}\left[\Psi(\tilde{\textbf{W}}^{avr}(t+1))\right] \le\textrm{ }& \mathbb{E}\left[\Psi(\tilde{\textbf{W}}^{avr}(t))\right] - \frac{\tilde{\eta}_t}{2}\left|\left|\nabla\Psi(\tilde{\textbf{W}}^{avr}(t))\right|\right|^2 + \frac{\varrho K \sigma^2}{BS}\tilde{\eta}_t^2 + (1+\varrho\tilde{\eta}_t)\Delta G(t)\label{eq_proofconv2_temp1}
\end{align}
where $\Delta G(t) = \tilde{\eta}_t\left(\left|\left|\sum_{s = 1}^{S}\nabla\Psi_s(\hat{\textbf{W}}_s(t)) - \nabla\Psi(\tilde{\textbf{W}}^{avr}(t))\right|\right|^2 + \left|\left|\sum_{s = 1}^{S}\hat{\nabla}\Psi_s(t) - \sum_{s = 1}^{S}\nabla\Psi_s(\hat{\textbf{W}}_s(t))\right|\right|^2\right)$.
Summing both sides of \eqref{eq_proofconv2_temp1} from $t = 0$ to $T-1$, we obtain
\begin{equation}\label{eq_proofconv2_main}
\frac{1}{\sum_{t=0}^{T-1}\tilde{\eta}_t}\sum_{t = 0}^{T-1}\tilde{\eta}_t\left|\left|\nabla\Psi(\tilde{\textbf{W}}^{avr}(t))\right|\right| \le \frac{2}{\sum_{t=0}^{T-1}\tilde{\eta}_t}\left(\Psi^{(0)}-\Psi^* + \frac{\varrho K \sigma^2}{BS}\sum_{t=0}^{T-1}\tilde{\eta}_t^2 + (1+\varrho\tilde{\eta}_t)\sum_{t=0}^{T-1}\Delta G(t)\right)
\end{equation}

We have
\begin{align*}
\left|\left|\sum_{s = 1}^{S}\nabla\Psi_s(\hat{\textbf{W}}_s(t)) - \nabla\Psi(\tilde{\textbf{W}}^{avr}(t))\right|\right|^2 =\textrm{ }& \left|\left|\sum_{s = 1}^{S}\left(\nabla\Psi_s(\hat{\textbf{W}}_s(t)) - \nabla\Psi_s(\tilde{\textbf{W}}^{avr}(t))\right)\right|\right|^2\\
\le\textrm{ }& S\sum_{s = 1}^{S} \left|\left|\nabla\Psi_s(\hat{\textbf{W}}_s(t)) - \nabla\Psi_s(\tilde{\textbf{W}}^{avr}(t))\right|\right|^2\\
\le\textrm{ }& S\varrho\sum_{s = 1}^{S} \left|\left|\nabla\Phi_s(\hat{\textbf{W}}_s(t)) - \nabla\Phi_s(\tilde{\textbf{W}}^{avr}(t))\right|\right| \left|\left|\hat{\textbf{W}}_s(t) - \tilde{\textbf{W}}^{avr}(t)\right|\right|
\end{align*}
Because $\left|\left|\nabla\Phi_s(\hat{\textbf{W}}_s(t)) - \nabla\Phi_s(\tilde{\textbf{W}}^{avr}(t))\right|\right| \le \left|\left|\nabla\Phi_s(\hat{\textbf{W}}_s(t))\right|\right| + \left|\left|\nabla\Phi_s(\hat{\textbf{W}}_s(t))\right|\right| \le 2\sigma, \forall s,t$, and $\left(\sum_{s = 1}^{S}\left|\left|\hat{\textbf{W}}_s(t) - \tilde{\textbf{W}}^{avr}(t)\right|\right|\right)^2 \le S\sum_{s = 1}^{S}\left|\left|\hat{\textbf{W}}_s(t) - \tilde{\textbf{W}}^{avr}(t)\right|\right|^2 = S\left|\left|\boldsymbol{\delta}(t)\right|\right|$, we have
\begin{equation}\label{eq_proofconv2_temp2}
\left|\left|\sum_{s = 1}^{S}\nabla\Psi_s(\hat{\textbf{W}}_s(t)) - \nabla\Psi(\tilde{\textbf{W}}^{avr}(t))\right|\right|^2 \le 2\sigma S^{3/2}\varrho \left|\left|\boldsymbol{\delta}(t)\right|\right|
\end{equation}
Consider $\left|\left|\sum_{s = 1}^{S}\hat{\nabla}\Psi_s(t) - \sum_{s = 1}^{S}\nabla\Psi_s(\hat{\textbf{W}}_s(t))\right|\right|^2$, we have
\begin{align*}
\left|\left|\sum_{s = 1}^{S}\hat{\nabla}\Psi_s(t) - \sum_{s = 1}^{S}\nabla\Psi_s(\hat{\textbf{W}}_s(t))\right|\right|^2 \le\textrm{ }& S\sum_{s = 1}^{S}\left|\left|\hat{\nabla}\Psi_s(t) - \nabla\Psi_s(\hat{\textbf{W}}_s(t))\right|\right|^2\\
=\textrm{ }& S\sum_{s = 1}^{S}\sum_{k = 1}^{K}\left|\left|\nabla_{\hat{\textbf{w}}_{s,k}}\Psi_s(\hat{\textbf{W}}_s(t-2K+2k)) - \nabla_{\hat{\textbf{w}}_{s,k}}\Psi_s(\hat{\textbf{W}}_s(t))\right|\right|^2\\
\le\textrm{ }& 2\sigma S\sum_{s = 1}^{S}\sum_{k = 1}^{K}\left|\left|\nabla\Psi_s(\hat{\textbf{W}}_s(t-2K+2k)) - \nabla\Psi_s(\hat{\textbf{W}}_s(t))\right|\right|\\
\le\textrm{ }& 2\sigma S\varrho\sum_{s = 1}^{S}\sum_{k = 1}^{K}\left|\left|\hat{\textbf{W}}_s(t-2K+2k) - \hat{\textbf{W}}_s(t)\right|\right|\\
\le\textrm{ }& 2\sigma S\varrho \sum_{k = 1}^{K} \sum_{s = 1}^{S}\left|\left|\hat{\textbf{W}}_s(t-2K+2k) - \hat{\textbf{W}}_s(t)\right|\right|\\
\le\textrm{ }& 2\sigma S\varrho \sum_{k = 1}^{K} \left(S\sum_{s = 1}^{S}\left|\left|\hat{\textbf{W}}_s(t-2K+2k) - \hat{\textbf{W}}_s(t)\right|\right|^2\right)^{1/2}\\
\le\textrm{ }& 2S^{3/2}\sigma\varrho\sum_{k = 1}^{K} \left|\left|\hat{\textbf{W}}(t-2K+2k) - \hat{\textbf{W}}(t)\right|\right|\\
\le\textrm{ }& 2S^{3/2}\sigma\varrho\sum_{k = 1}^{K}\left(\left(t-1-\Delta_{t,k}\right)^{1/2}\sum_{\tau = \Delta_{t,k}}^{t-1}\left|\left|\hat{\textbf{W}}(\tau+1) - \hat{\textbf{W}}(\tau)\right|\right|\right)\\
\le\textrm{ }& (8S^3K)^{1/2}\sigma\varrho\sum_{k = 1}^{K}\left(\sum_{\tau = \Delta_{t,k}}^{t-1}\left|\left|\hat{\textbf{W}}(\tau+1) - \hat{\textbf{W}}(\tau)\right|\right|\right)
\end{align*}
In addition, we have
\begin{align*}
\left|\left|\hat{\textbf{W}}(\tau+1) - \hat{\textbf{W}}(\tau)\right|\right| =\textrm{ }& \left|\left|\boldsymbol{\delta}(\tau+1) + \left(\textbf{1}_S \otimes \textbf{I}_d\right) \tilde{\textbf{W}}^{avr}(\tau+1) - \left(\textbf{1}_S \otimes \textbf{I}_d\right) \tilde{\textbf{W}}^{avr}(\tau) + \boldsymbol{\delta}(\tau)\right|\right|\\
\le\textrm{ }& \left|\left|\boldsymbol{\delta}(\tau+1)\right|\right| + \left|\left|\boldsymbol{\delta}(\tau)\right|\right| + \tilde{\eta}_{\tau}\sigma\sqrt{\frac{K}{BS}}
\end{align*}
Because $\tilde{\eta}_{\tau+1} \le \tilde{\eta}_{\tau}, \forall \tau \ge 0$
\begin{align*}
\Delta G(t) &\le 2\sigma S^{3/2}\varrho \left|\left|\boldsymbol{\delta}(t)\right|\right|\tilde{\eta}_t + (8S^3K)^{1/2}\varrho \sum_{k = 1}^{K}\left(\sum_{\tau = \Delta_{t,k}}^{t-1}\left(\left|\left|\boldsymbol{\delta}(\tau+1)\right|\right|\tilde{\eta}_{\tau+1} + \left|\left|\boldsymbol{\delta}(\tau)\right|\right|\tilde{\eta}_{\tau} + \sigma\sqrt{\frac{K}{BS}}\tilde{\eta}_{\tau+1}\tilde{\eta}_{\tau}\right)\right).
\end{align*}
Define \[\Theta(T) = \sum_{t = 0}^{T-1}\sum_{k = 1}^{K}\sum_{\tau = \Delta_{t,k}}^{t-1}\tilde{\eta}_{\tau+1}\tilde{\eta}_{\tau},\] 
\[\Omega(T) = \sum_{t = 0}^{T-1}\left(\sum_{k = 1}^{K}\sum_{\tau = \Delta_{t,k}}^{t-1}\left(\left|\left|\boldsymbol{\delta}(\tau+1)\right|\right|\tilde{\eta}_{\tau+1} + \left|\left|\boldsymbol{\delta}(\tau)\right|\right|\tilde{\eta}_{\tau}\right) + \left|\left|\boldsymbol{\delta}(t)\right|\right|\tilde{\eta}_{t}\right).\]
We have
\begin{align*}
\sum_{t = 0}^{T-1}\sum_{k = 1}^{K}\sum_{\tau = \Delta_{t,k}}^{t-1}\tilde{\eta}_{\tau+1}\tilde{\eta}_{\tau} &\le \frac{1}{2}\sum_{t = 0}^{T-1}\sum_{k = 1}^{K}\sum_{\tau = \Delta_{t,k}}^{t-1}\left(\tilde{\eta}_{\tau+1}^2 + \tilde{\eta}_{\tau}^2\right)\\
&= \frac{1}{2} \sum_{k = 1}^{K}\sum_{\tau = \Delta_{T,k}}^{T-1}\sum_{t = 0}^{\tau-1} \left(\tilde{\eta}_{t+1}^2 + \tilde{\eta}_{t}^2\right)\\
&\le \frac{1}{2}K(2K)\sum_{t = 0}^{T-1} \left(\tilde{\eta}_{t+1}^2 + \tilde{\eta}_{t}^2\right)\\
&\le \frac{1}{2}K(2K) 2\frac{M_0}{S^2}\\
&= 2K^2\frac{M_0}{S^2}
\end{align*}
So we have $\Theta(T) \le 2K^2\frac{M_0}{S^2}$.

Now we find the bounds for $\Omega(T)$.
\begin{align*}
\Omega(T) &= \sum_{t = 0}^{T-1}\sum_{k = 1}^{K}\sum_{\tau = \Delta_{t,k}}^{t-1}\left(\left|\left|\boldsymbol{\delta}(\tau+1)\right|\right|\tilde{\eta}_{\tau+1} + \left|\left|\boldsymbol{\delta}(\tau)\right|\right|\tilde{\eta}_{\tau}\right) + \sum_{t = 0}^{T-1}\left|\left|\boldsymbol{\delta}(t)\right|\right|\tilde{\eta}_{t}\\
&= \sum_{k = 1}^{K}\sum_{\tau = \Delta_{T,k}}^{T-1}\sum_{t = 0}^{\tau-1}\left(\left|\left|\boldsymbol{\delta}(t+1)\right|\right|\tilde{\eta}_{t+1} + \left|\left|\boldsymbol{\delta}(t)\right|\right|\tilde{\eta}_{t}\right) + \sum_{t = 0}^{T-1}\left|\left|\boldsymbol{\delta}(t)\right|\right|\tilde{\eta}_{t}\\
&\le K(2K)2\sum_{t = 0}^{T}\left|\left|\boldsymbol{\delta}(t)\right|\right|\tilde{\eta}_{t} + \sum_{t = 0}^{T-1}\left|\left|\boldsymbol{\delta}(t)\right|\right|\tilde{\eta}_{t}\\
&\le \left(4K^2+1\right)\sum_{t = 0}^{T}\left|\left|\boldsymbol{\delta}(t)\right|\right|\tilde{\eta}_{t}
\end{align*}
In addition, we have
\begin{align*}
\left|\left|\boldsymbol{\delta}(t)\right|\right|\tilde{\eta}_t \le \frac{\gamma^{t}}{S}\left|\left|\boldsymbol{\delta}(0)\right|\right|\eta_t + \frac{1}{S}\sum\limits_{\tau = 0}^{t-1}\gamma^{t-\tau}\eta_{\tau}\eta_t \le \frac{\gamma^{t}}{S}\left|\left|\boldsymbol{\delta}(0)\right|\right|\eta_t + \frac{1}{S}\sum\limits_{\tau = 0}^{t-1}\gamma^{t-\tau}\left(\frac{1}{2}\eta_t^2 + \frac{1}{2}\eta_{\tau}^2\right)
\end{align*}
Note that
\begin{align*}
\sum\limits_{t = 0}^{T}\gamma^{t}\left|\left|\boldsymbol{\delta}(0)\right|\right|\eta_t \le \eta_0\left|\left|\boldsymbol{\delta}(0)\right|\right|\sum\limits_{t = 0}^{T}\gamma^{t} \le \frac{\eta_0\left|\left|\boldsymbol{\delta}(0)\right|\right|}{1-\gamma}
\end{align*}
\begin{align*}
\sum\limits_{t = 0}^{T}\sum\limits_{\tau = 0}^{t-1}\gamma^{t-\tau}\left(\frac{1}{2}\eta_t^2 + \frac{1}{2}\eta_{\tau}^2\right) \le\textrm{ }& \left(2\sum\limits_{t = 1}^{T}\gamma^t\right)\left(\sum\limits_{t = 0}^{T}\eta_t^2\right) \le \frac{2\gamma M_0}{1-\gamma}
\end{align*}
Thus $\sum\limits_{t = 0}^{T} \Delta G(t) \le 4\sqrt{\frac{2K}{S}}\varrho K^2 M_0 + \left((8S^3K)^{1/2}\left(4K^2+1\right)\varrho + 2\sigma S^{3/2}\varrho\right)\left(\frac{\eta_0\left|\left|\boldsymbol{\delta}(0)\right|\right|}{1-\gamma} + \frac{2\gamma M_0}{1-\gamma}\right) < \infty$.
Then \eqref{eq_proofconv2_main} guarantees that
\begin{equation}\label{eq_proofconv2_temp3}
\lim\limits_{T \rightarrow \infty} \frac{1}{\sum_{t=0}^{T-1}\tilde{\eta}_t}\sum_{t = 0}^{T-1}\tilde{\eta}_t\left|\left|\nabla\Psi(\tilde{\textbf{W}}^{avr}(t))\right|\right| = 0.
\end{equation}
Note that \[\frac{1}{\sum_{t=0}^{T-1}\tilde{\eta}_t}\sum_{t = 0}^{T-1}\tilde{\eta}_t\left|\left|\nabla\Psi(\tilde{\textbf{W}}^{avr}(t))\right|\right| = \frac{1}{\sum_{t=0}^{T-1}\eta_t}\sum_{t = 0}^{T-1}\eta_t\left|\left|\nabla\Psi(\tilde{\textbf{W}}^{avr}(t))\right|\right|.\]
As $\tilde{\textbf{W}}^{avr}(\tau)$ is chosen randomly from $\{\tilde{\textbf{W}}^{avr}(t): 0 \le t \le T-1\}$ with probabilities proportional to $\{\eta_0, \eta_1, \dots, \eta_{T-1}\}$, we have $\mathbb{E}\left[\tilde{\textbf{W}}^{avr}(\tau)\right] = \frac{1}{\sum_{t=0}^{T-1}\eta_t}\sum_{t = 0}^{T-1}\eta_t\left|\left|\nabla\Psi(\tilde{\textbf{W}}^{avr}(t))\right|\right|$.
So \eqref{eq_proofconv2_temp3} implies \eqref{eq_theorem2_2}.
\section{Proposed distributed training algorithm}
\begin{algorithm}[htb]
   \caption{The proposed distributed training method}
   \label{alg_proposed}
\begin{algorithmic}
   \STATE {\bfseries Input:}\\Weights $\textbf{W}(0) = \left[\begin{matrix}(\hat{\textbf{W}}_1(0))^T & \cdots & (\hat{\textbf{W}}_S(0))^T\end{matrix}\right]^T$\\
   where $\hat{\textbf{W}}_s(0) = \left[\begin{matrix}(\hat{\textbf{w}}_{s,1}(0))^T & \cdots & (\hat{\textbf{w}}_{s,K}(0))^T\end{matrix}\right]^T$ for all $1 \le s \le S$.\\
   Step-size sequences $\{\eta_0, \eta_2, \dots, \eta_{T-1}\}$  
   \FOR{$t=0$ {\bfseries to} $T-1$: agent $(s,k), \forall 1 \le s \le S, 1 \le k \le K$ {\bfseries in parallel}}
   \IF{$k = 1$}
   \STATE sample a mini-batch $\mathcal{B}_s(t) \subset \mathcal{D}_s$
   \ENDIF
   \STATE compute $\textbf{h}_{l}^{(s,k)}(t-k+1)$ for all $p_k \le l \le q_k$
   \STATE compute $\frac{\partial \phi(\boldsymbol{\chi}^{(n)},\tilde{\textbf{W}}_s(t-2K+k+1))}{\partial \hat{\textbf{w}}_{s,k}}$ for all $\boldsymbol{\chi}^{(n)} \in \mathcal{B}(t-2K+k+1)$
   \IF{$k > 1$}
   \STATE send $\frac{\partial \phi(\boldsymbol{\chi}^{(n)},\tilde{\textbf{W}}_s(t-2K+k+1))}{\partial \textbf{w}_{p_{k}}}$ to $(s,k-1)$ for all $\boldsymbol{\chi}^{(n)} \in \mathcal{B}(t-2K+k+1)$
   \STATE receive $\textbf{h}_{q_{k-1}}^{(s,k-1)}(t-k+2)$ from $(s,k-1)$
   \ENDIF
   \IF{$k < K$}
   \STATE send $\textbf{h}_{q_k}^{(s,k)}(t-k+1)$ to $(s,k+1)$
   \STATE receive $\frac{\partial \phi(\boldsymbol{\chi}^{(n)},\tilde{\textbf{W}}_s(t-2K+k+2))}{\partial \textbf{w}_{p_{k+1}}}$ from $(s,k+1)$ for all $\boldsymbol{\chi}^{(n)} \in \mathcal{B}(t-2K+k+2)$
   \ENDIF
   \STATE $\hat{\textbf{u}}_{s,k}(t) \gets$ (\ref{eq_dl_update}a)
   \STATE send $\hat{\textbf{u}}_{s,k}(t)$ to $(r,k) \in \mathcal{N}_{s,k}$
   \STATE receive $\hat{\textbf{u}}_{r,k}(t)$ from $(r,k) \in \mathcal{N}_{s,k}$
   \STATE $\hat{\textbf{w}}_{s,k}(t) \gets$ (\ref{eq_dl_update}b)
   \ENDFOR    
\end{algorithmic}
\end{algorithm}
\end{document}